\documentclass[twoside]{article}

\usepackage[accepted]{aistats2022}
\usepackage[round]{natbib}

\usepackage[utf8]{inputenc} 
\usepackage[T1]{fontenc}    
\usepackage{hyperref}       
\usepackage{url}            
\usepackage{booktabs}       
\usepackage{amsfonts}       
\usepackage{nicefrac}       
\usepackage{microtype}      
\usepackage{xcolor}         

\usepackage{lmodern}

\usepackage{graphicx}
\usepackage{caption}
\usepackage{subcaption}

\usepackage{amsmath,amsthm,amssymb}
\usepackage[linesnumbered,vlined,ruled]{algorithm2e}

\renewcommand{\Re}{\mathbb{R}}
\newcommand{\Na}{\mathbb{N}}
\newcommand{\ex}[1]{\mathbb{E}\left[#1\right]}

\newtheorem{theorem}{\bf Theorem}
\newtheorem{remark}{\bf Remark}
\newtheorem{lemma}[theorem]{\bf Lemma}
\newtheorem{corollary}[theorem]{\bf Corollary}
\newtheorem{assumption}{Assumption}

\newcount\comments  
\newcommand{\genComment}[2]{\ifnum\comments=1{\color{#1}{\textsf{\footnotesize #2}}}\fi}

\DeclareMathOperator*{\argmin}{arg\,min} %

\newcommand{\Real}{\mathbb{R}}

\begin{document}
	
	%
	
	%
	\runningauthor{Yuntian Deng, Xingyu Zhou,  Baekjin Kim, Ambuj Tewari, Abhishek Gupta, Ness Shroff}
	
	\twocolumn[
	\aistatstitle{Weighted Gaussian Process Bandits for Non-stationary Environments}
	
	\aistatsauthor{ Yuntian Deng \And Xingyu Zhou \And  Baekjin Kim}
	
	\aistatsaddress{ The Ohio State University \And  Wayne State University \And University of Michigan}
	
	\aistatsauthor{ Ambuj Tewari \And Abhishek Gupta \And Ness Shroff }
	
	\aistatsaddress{ University of Michigan \And  The Ohio State University \And The Ohio State University} ]

	\begin{abstract}
		In this paper, we consider the Gaussian process (GP) bandit optimization problem in a non-stationary environment. To capture external changes, the black-box function is allowed to be time-varying within a reproducing kernel Hilbert space (RKHS). To this end, we develop WGP-UCB, a novel UCB-type algorithm based on weighted Gaussian process regression. A key challenge is how to cope with infinite-dimensional feature maps. To that end, we leverage kernel approximation techniques to prove a sublinear regret bound, which is the first (frequentist) sublinear regret guarantee on weighted time-varying bandits with general nonlinear rewards. This result  generalizes both non-stationary linear bandits and standard GP-UCB algorithms. Further, a novel concentration inequality is achieved for weighted Gaussian process regression with general weights. We also provide universal upper bounds and weight-dependent upper bounds for weighted maximum information gains. These results are of independent interest for applications such as news ranking and adaptive pricing, where weights can be adopted to capture the importance or quality of data. Finally, we conduct experiments to highlight the favorable gains of the proposed algorithm in many cases when compared to existing methods.
	\end{abstract}

	
	
	
	\section{Introduction}
	There has been significant interest in developing the theoretical foundations and practical algorithms for solving bandit optimization problems. This interest has been driven by many practical applications, where one needs to sequentially select query points to maximize the cumulative reward \citep{bubeck2012regret}. Such sequential decision-making is usually  based on noisy feedback from black-box functions defined over a possibly large domain space.

	The most classical model is the multi-armed bandit (MAB) \citep{robbins1952some} where query points are independent and finite. It is then extended to stochastic linear bandits \citep{auer2002using, abbasi2011improved}, where the black-box function is linear, and query points become non-orthonormal.
	Gaussian process bandits \citep{srinivas2009gaussian, chowdhury2017kernelized} further generalizes the previous two by allowing general black-box functions (e.g., non-linear and non-convex) by utilizing the representation power of the reproducing kernel Hilbert space (RKHS). These models of online decision-making have becoming ubiquitous in practical applications, such as change-points detection \citep{liu2018change}, personalized news recommendation \citep{li2010contextual}, and portfolio selection \citep{huo2017risk}.
	
	In real-world scenarios, the unknown function is often not fixed but varies over time. For example,  the channel conditions in wireless networks are time-varying, and thus the Quality of Service is not static \citep{zhou2019non}. In recommender systems, users' preferences may change with growth, and the corresponding reward function for any recommending action is time-varying \citep{li2010contextual}.  This has motivated recent studies in online-decision making under time-varying environments, such as in MAB \citep{besbes2014stochastic}, linear bandits \citep{kim2020randomized}, and Gaussian process bandits \citep{bogunovic2016time}. Roughly speaking, there are three commonly used techniques to handle non-stationarity -- \emph{restarting, sliding window }and \emph{weighted penalty}. By restarting, the learning agent resets the learning process once in a while to directly discard all the old information at once (e.g., \citep{zhao2020simple, besson2019generalized} ) while under the sliding window, the learning agent gradually discard old information by only using the most recent data in the learning process (e.g., \citep{cheung2019learning}). Recently, \cite{russac2019weighted} proposes the weighted penalty approach, which puts more weight on the most recent data while penalizing outdated information via less weight. This approach can be viewed as a `soft' way of discounting outdated information rather than completely dropping it as in restarting and sliding window methods. The weighted penalty approach has been shown to be beneficial when the black-box function is linear \citep{russac2019weighted} or general continuously differentiable Lipschitz function \citep{russac2020algorithms}. However, it remains an open problem whether one can achieve the advantages of the weighted penalty approach for general black-box functions, in particular the ones in an RKHS, which enjoys the uniform approximation of an arbitrary continuous function~\citep{micchelli2006universal} (under a proper choice of the kernel). 
	
	Recently, \cite{wei2021non} provides optimal results for (generalized) linear bandits, via maintaining different instances of base algorithms. It remains an open problem whether Gaussian Process bandit can achieve its regret bound. To be specific, we do not know whether GP bandits satisfies Assumption 1 in this paper and what is the form of $C(t)$ and $\Delta(t)$ for GP bandits. If $C(t)$ does not have the same polynomial form  as its Theorem 2, its result cannot be extended to GP bandits. 
	
	In this paper, we take the first step to tackling this fundamental problem by proposing a novel weighted penalty algorithm with rigorous regret guarantees in the context of Gaussian process bandits. This is achieved by overcoming several key challenges. First, to fully utilize the representation power of an RKHS for general functions, our choice of kernel often has an infinite-dimensional feature space (e.g., Squared Exponential kernel). In this case, all existing regret analysis breaks down as  regret bounds in these works have an explicit, growing
	dependence on the feature dimension (e.g., $d$). Moreover, the standard approach of resolving the dependence on $d$ in the regret bounds for GP bandits does not apply in our case. This is because we are dealing with a weighted GP regression rather than a standard one, which directly raises three substantial challenges. First, we need to find a new self-normalized concentration inequality to show that the posterior mean under the weighted GP regression is still close to the true function in a certain sense, which helps to translate the cumulative regret into a sum of predictive variance. Then,  we need to find a new technique to bounding this term by a properly defined so-called Maximum Information Gain (MIG). Finally, existing bounds on MIG also do not apply, and thus we have to derive the new ones in our weighted setting.
	
	\textbf{Contributions.} In summary, our contributions can be summarized as follow. 
	
	First, we develop a general framework for the regret analysis under weighted GP regression by overcoming the aforementioned challenges. In particular, for general weighted GP regression, we establish the first self-normalized concentration inequality. Then, by novel applications of Quadrature Fourier features (QFF) approximation and Mercer's theorem, we present the first bounds for the sum of predictive variance and the corresponding MIGs in the weighted case. These results are not only the cornerstones in our setting, but also could be useful for other general GP regression settings. 
	
	Second, we propose a new algorithm - Weighted Gaussian Process Upper Confidence Bound (WGP-UCB) for non-stationary bandit optimization. It generalizes the standard GP-UCB algorithms \citep{srinivas2009gaussian, chowdhury2017kernelized} in stationary environments to the time-varying case. This is also a significant generalization of discounted linear bandit \citep{russac2019weighted} and discounted generalized linear UCB \citep{russac2020algorithms} by allowing the payoff function to be within a much broader class of functions (thanks to the use of RKHS).
	
	Third, by a proper choice of the weighted scheme in WGP-UCB, we establish the first regret bound for the weighted penalty algorithm in the context of GP bandits by utilizing our novel results for the weighted GP regression. In particular, we have a regret bound $O(\dot \gamma_T^{7/8} B_T^{1/4} T^{3/4})$ if the variation budget $B_T$ is known, and a regret bound $O(\dot \gamma_T^{7/8} B_T T^{3/4})$ if $B_T$ is unknown, for both abruptly-changing and slowly-varying environments where $\dot \gamma_T$ is a properly defined maximum information gain. Note that this result directly recovers the existing weighted penalty results as special cases (e.g., a choice of the linear kernel). 
	
	
	
	
	\textbf{Related Work.}
	Online learning in changing environments has been well studied. In traditional MAB, \citep{auer2019adaptively} considers the situation where reward distributions may change abruptly several times, which is usually referred to switching bandits \citep{garivier2011upper} or abruptly-changing environments. The regret bound usually depends on the number of changes, which relies on change-point detection \citep{cao2019nearly,liu2018change}. An alternative approach to quantify time-variations is variation budget \citep{besbes2014stochastic, besbes2015non, besbes2019optimal}, which captures the cumulative temporal variation of system parameters for the total time horizon.

	In the stochastic linear bandits setting, we recall that there are mainly three strategies to deal with non-stationarity : \emph{restarting} \citep{zhao2020simple}, \emph{sliding window} \citep{cheung2019learning}, and \emph{weighted penalty}.
	The last strategy leverages an increasing weight sequence to emphasize the impact of recent observations while gradually forgetting past observations. In \cite{russac2019weighted}, exponentially increasing weights are used to develop the D-LinUCB algorithm based on the weighted least square estimator. Two variants of this algorithm are developed in \cite{kim2020randomized} based on perturbation techniques. Recently, a technical flaw in these three works \citep{cheung2019learning, russac2019weighted, zhao2020simple} was identified in \cite{zhao2021non}, which corrects the order of regret bounds in all three algorithms. For generalized linear models, sliding window and weighted penalty algorithms are developed in \cite{russac2020algorithms}, where the payoff functions are required to be continuously differentiable and Lipschitz.
	
	
	For the Gaussian process bandits, there are two different assumptions on the black-box functions. The Bayesian setting assumes that the unknown function is a sample from a GP with a known kernel, while the frequentist setting (agnostic setting in \cite{srinivas2009gaussian}) assumes that the unknown function is a fixed function in a reproducing kernel Hilbert space (RKHS) with bounded norm. Under the Bayesian setting, \cite{bogunovic2016time} proposes a 
	discounted algorithm and a restarting algorithm, assuming that the evolution of Gaussian process obeys a simple Markov model. Under the frequentist setting, \cite{zhou2021no} introduces restarting and sliding window algorithms and obtains regret bounds based on maximum information gain. However, due to difficulties arising from the time variation and infinite-dimensional feature maps, the weighted penalty algorithm has not been studied yet under the frequentist setting, which is also the future direction as listed in \cite{bogunovic2016time}.

	\section{Problem Statement and Preliminaries}

	In this section, we introduce the setting of our problem and necessary preliminaries.
	
	We consider the non-stationary problem of sequentially maximizing reward function $f_t: D \xrightarrow{} \Re$ over a set of decisions $D \subset \Re^d$ . At each discrete time slot $t=1,2,\ldots$, the learning agent selects an action (query point) $x_t \in D$ and the reward $f_t(x_t)$ is observed through a noisy channel as $y_t=f_t(x_t)+\epsilon_t$
	where  $\epsilon_t$ is the zero mean noise. Denote the history as $\mathcal H_{t-1} = \{ (x_s, y_s): s \in \{1,2,\ldots, t-1\}\}$. Conditioned on history $\mathcal H_{t-1}$, the noise sequence $\epsilon_t$ is $R$-sub-Gaussian for a fixed constant $R\geq 0 $, i.e. $\forall t>1, \forall \lambda \in \Re, \ex{ e^{\lambda \epsilon_t}|\mathcal F_{t-1} } \leq \text{exp}(\frac{\lambda^2 R^2}{2})$
	where $\mathcal F_{t-1} = \sigma \left( \mathcal H_{t-1}, x_t \right)$ is the $\sigma$-algebra generated by actions and rewards observed so far. 
	
	The objective of the learning agent is to maximize the cumulative reward $\sum_{t=1}^T f_t(x_t)$. This is equivalent to  minimize its \emph{dynamic regret} $R_T$, which  is defined as $R_T = \sum_{t=1}^T f_t(x_t^*) - f_t(x_t)$ where $x_t^*= \arg \max _{x\in D} f_t(x)$ is the attainable best action at time $t$ for function $f_t(\cdot)$.

	\textbf{Regularity Assumptions}: We assume that $f_t$ is a fixed function in a Reproducing Kernel Hilbert Space (RKHS) with a bounded norm. Specifically, we assume that $D$ is compact. The RKHS, denoted by $H_k(D)$, is completely specified by its kernel function $k(\cdot, \cdot)$, with an inner product $\left< \cdot, \cdot \right>_{H}$ satisfying the reproducing property: $f(x) = \left <f, k(x, \cdot) \right >_H$ for all $f \in H_k(D)$. The RKHS norm is given by $\|f\|_H := \sqrt{ \left< f, f \right>_H }$. We assume that $f_t$ at each time $t$ is bounded by $\|f_t\|_H \leq B$ for a fixed constant $B$. Moreover, we assume a bounded variance by restricting $k(x,x) \leq 1$.  The assumptions hold for practically relevant kernels.
	One concrete example is Squared Exponential kernel, defined as 
	$k_{SE}(x,x') = \text{exp}(-s^2/2l^2)$
	where scale parameter $l>0$ and $s=\|x-x'\|_2$ specifies distance between two points.

	
	\textbf{Time-varying Budget}: As the environment is time-varying, we assume that the total variation of $f_t$ satisfies the following budget, $\sum_{t=1}^{T-1} \|f_{t+1}-f_t\|_H \leq B_T$, including both abruptly-changing and slowly-changing environments.
	
	\textbf{Maximum Information Gain}: We use $I(y_A;f_A)$ to denote the mutual information between $f_A=[f(x)]_{x\in A}$ and $y_A = f_A + \epsilon_A$, which quantifies the reduction in uncertainty about $f$ after observing $y_A$ at points $A \subset D$. Then the maximum information gain \citep{srinivas2009gaussian} is defined as,  $\gamma_n := \max_{A\subset D: |A|=n} I(y_A;f_A) = \max_{A\subset D: |A|=n} \frac{1}{2} \log \det(I + \lambda^{-1} K_A)$, where $K_A= [k(x,x')]_{x,x'\in A}$.

	
	\textbf{Agnostic setting}: We recall the agnostic setting in standard GP-UCB algorithm \citep{chowdhury2017kernelized} for the stationary environment.  Gaussian process (GP) and Gaussian likelihood models are used to design this algorithm. $GP_D(0,k(\cdot, \cdot ))$ is the prior for reward function $f_t$. The noise $\epsilon_t$ is drawn independently from $\mathcal N(0,\lambda)$. Conditioned on the history $\mathcal H_t$, it has the posterior distribution of $f_t$, $GP_D \left(\mu_t(\cdot),\sigma_t^2(\cdot) \right)$, where the posterior mean and variance are defined as
	\begin{align}
	\mu_t(x)  &= k_t(x)^T (K_t + \lambda I )^{-1} y_{1:t} \label{eq:mean_var}\\
	\sigma^2_t(x) &= k(x,x) -k_t(x)^T (K_t + \lambda I )^{-1} k_t(x)  \label{eq:sta_var}
	\end{align} 
	where $y_{1:t}\in \Re^t$ is the reward vector $[y_1,\ldots,y_t]^T$. For set of sampling points $A_t=\{x_1, \ldots, x_t\}$, the kernel matrix is $K_t = [k(x,x')]_{x,x' \in A_t} \in \Re^{t \times t}$ and the vector $k_t(x)=[k(x_1,x), \ldots, k(x_t,x)]^T \in \Re^t$.
	The GP prior and Gaussian likelihood are only used for algorithm design and do not affect the setting of reward function $f_t \in H_k(D)$ and noise $\epsilon_t$ (i.e., could be sub-Gaussian).

	\section{Weighted Gaussian Process Regression}
	\label{sec:WGP}
	In this section, we introduce a general weighted algorithm based on weighted GP regression. The key difference with standard GP regression is that we allow different weight for each data point. It is worth noting that this result is fairly generic in the sense that it can be applied in general situations where weights are used to associate with `importance' in the data points. E.g., more weights are assigned to observations that are less noisy in weighted ridge regression \citep{zhou2021nearly}.
	
	In particular, the weighted GP regression under a changing regularizer is defined by 
	\begin{align*}
	\hat{f} =  &\argmin_{f \in H_k(D)} \sum_{s=1}^{t-1} w_s (y_s - f(x_s) )^2 + {\lambda_t} \| f \|_{\mathcal H}^2
	\end{align*}
	where each data point is associated with a weight in computing the least square estimate. Due to this, standard posterior mean and variance in~\eqref{eq:mean_var} and \eqref{eq:sta_var} fail to capture the statistics of $\hat{f}$. To this end, we have to carefully adjust the kernel vector and kernel matrix in~\eqref{eq:mean_var} by incorporating proper weights. Specifically, 
	let $W=\text{diag}(\sqrt{w_1}, \sqrt{w_2}, \ldots, \sqrt{w_t})  \in \Re^{t\times t}$.
	Then we define the weighted version of kernel matrix $\tilde K_t := W K_t W^T $ and weighted kernel vector $\tilde k_t(x) :=  W k_t(x)$. We further define weighted observation $\tilde y_{1:t} := W y_{1:t}= [\sqrt{w_1} y_1, \ldots, \sqrt{w_t} y_t]^T$. Finally, a weight-dependent regularizer is defined by $ \lambda_t = \lambda w_t$. Then,  $\hat{f}$ and its uncertainty are given by the following equations, respectively.
	\begin{align}
	\tilde \mu_t(x)  &= \tilde k_t(x)^T (\tilde K_t + \lambda_t I_t )^{-1} \tilde y_{1:t} \label{equ:mean}\\
	\tilde{\sigma}^2_t(x) &=  k(x,x) -  \tilde  k_t(x)^T (\tilde K_t + \lambda_t I_t )^{-1} \tilde  k_t(x)
	\label{equ:variance}
	\end{align}
	One can see that~\eqref{equ:mean} and \eqref{equ:variance} share the same structure as the standard ones in~\eqref{eq:mean_var} and \eqref{eq:sta_var}. This nice result directly enables us to design a UCB-type learning algorithm in the weighted case as follows. 
	
	\textbf{Algorithm}: The Weighted Gaussian Process-UCB algorithm (WGP-UCB) (Algorithm~\ref{alg:weighted}) uses a combination of the weighted posterior mean $\tilde \mu_{t-1}(x)$ and weighted standard deviation $\tilde \sigma_{t-1}(x)$ to construct an upper confidence bound (UCB) over the unknown function. It then chooses an action $x_t$ at time $t$ as follows:
	\begin{align} \label{equ:x_t}
	x_t := \arg \max_{x \in D} \tilde \mu_{t-1} (x) + \beta_{t-1} \tilde \sigma_{t-1} (x)
	\end{align}
	where $ \beta_t =   B+  \frac{1}{\sqrt{\lambda}}  R \sqrt{  2\log(\frac{1}{\delta}) +  2 \bar \gamma_t}$ and $0 < \delta < 1$. We note that this algorithm enjoys the same simplicity as the standard non-weighted one~\citep{chowdhury2017kernelized,srinivas2009gaussian}. Meanwhile, there are substantial differences. In particular, besides the new posterior mean and variance, we also need to replace the MIG in the confidence width $\beta_t$ by a weighted one, i.e., $\bar{\gamma}_t$. This term is defined as $\bar \gamma_t = \max_{A\subset D: |A|=t} \frac{1}{2} \log \det (I + \alpha_t^{-1} W^2 K_A W^{2T} )= \max_{A\subset D: |A|=t} \frac{1}{2} \log \det(I + \alpha_t^{-1} \bar K_t )$, where $W^2=\text{diag}(w_1, w_2, \ldots, w_t)  \in \Re^{t\times t}$, $\bar K_t= W^2 K_t W^{2T} $ is the double-weighted kernel matrix and $\alpha_t = \lambda w_t^2$.
	

	
	In the following sections, we will develop a general framework for the regret analysis in the weighted GP regression, which recovers the standard analysis as special cases by choosing $w_s = 1$ for all $s\in[T]$~\citep{chowdhury2017kernelized,srinivas2009gaussian}. From a high-level perspective, the typical recipe of deriving regret bounds in GP bandits has three main steps. \textbf{(I)} One needs to first show that the true underlying function is close to the posterior mean within some distance given by the standard derivation. This concentration result is the cornerstone and is typically achieved by relying on the so-called self-normalized inequality. 
	\textbf{(II)} Based on this concentration, one can bound the cumulative regret by a sum of predictive variance terms. This can be further upper bounded by the MIG. \textbf{(III)} The MIG will finally be upper bounded depending on the choice of kernels, which leads to the final regret bound. However, all the three key steps face new challenges in our weighted case, and hence we will conquer them one by one.

	
	

	\begin{algorithm}[t]
		\SetAlgoLined
		\SetKwInOut{Input}{Input}\SetKwInOut{Initialization}{Initialization}
		\Input{parameters $k(\cdot, \cdot),B,R,\lambda,\delta$, weights $\{\omega_t\}_{t=1}^T$.} 
		\BlankLine
		\For{$t\geq 1$}{
			Set $ \beta_{t-1} =   B+  \frac{1}{\sqrt{\lambda}}  R \sqrt{  2\log(\frac{1}{\delta}) +  2 \bar \gamma_{t-1}}$  \;
			Choose $x_t = \arg \max_{x \in D} \tilde \mu_{t-1} (x) + \beta_{t-1} \tilde \sigma_{t-1} (x)$\;
			Observe reward $y_t = f_t(x_t) + \epsilon_t$\;
			Update $\tilde \mu_t(x)$ and $\tilde \sigma_t(x)$ according to Equation \eqref{equ:mean} and \eqref{equ:variance}.
		}
		\caption{Weighted Gaussian Process UCB (WGP-UCB)}{\label{alg:weighted}}
	\end{algorithm}

	\section{Confidence Bounds}
	In this section, we focus on deriving a new concentration inequality in the weighted case to show that the new posterior mean is still close to the true function in the non-stationary environment, which resolves the challenge \textbf{(I)} listed above. In particular, we present concentration results for both stationary and non-stationary environments. 
	
	
	
	To start with, we introduce a particular feature map via Mercer's Theorem, which will only be used in our analysis. The following version of Mercer's theorem (described by Theorem \ref{thm: feature} next) is adapted from Theorem 4.1 and 4.2 in \cite{kanagawa2018gaussian}, which roughly says that the kernel function can be expressed in terms of the eigenvalues and eigenfunctions under mild conditions.
	\begin{theorem} \label{thm: feature}
		Let $\mathcal{X}$ be a compact metric space, $k:\mathcal{X} \times \mathcal{X} \to \Real$ be a continuous kernel with respect to a finite Borel measure $\nu$ whose support is $\mathcal{X}$. Then, there is a countable sequence $(\lambda_i,\phi_i)_{i\in \mathbb{N}}$, where $\lambda_i \ge 0$ and $\lim_{i\to \infty}\lambda_i = 0$ and $\{\phi_i\}$ forms an orthonormal basis of $L_{2,\nu}(\mathcal{X})$, such that 
		\begin{align}  \label{equ: kernel}
		k(x,x')=\sum_{m=1}^\infty c_m \phi_m(x) \phi_m(x')
		\end{align}
		where $c_m \in \Re^+$ and $\phi_m \in \mathcal{H}$ for $m \geq 1$. $\{c_m\}_{m=1}^\infty$ is the eigenvalue sequence in decreasing order.
		$\{\phi_m\}_{m=1}^\infty$ are the eigenfeatures (eigenfunctions) of $k$. The RKHS can also be represented in terms of $\{c_m, \phi_m \}_{m=1}^\infty$. i.e., 
		\begin{align*}
		\mathcal{H} = \big\{f(\cdot)=\sum_{m=1}^\infty \theta_m \sqrt{c_m} \phi_m (\cdot):
		\|f\|_H : = \|\theta\|_2  < \infty  \big\}.
		\end{align*}
		
		
	\end{theorem}

	
	Based on this theorem, we can explicitly define a feature map as $\varphi(x) = [\varphi_1(x),\varphi_2(x), \ldots]^T  \in \Re^M$ ($M$ may be infinity) where $\varphi_m = \sqrt{c_m} \phi_m \in \mathcal{H}$ and $\varphi_m:= D \rightarrow{} \Re$.
	Given $\theta=[\theta_1, \theta_2, \ldots]^T \in \Re^M$, we have reward function $f(x)=\theta^T \varphi(x)$ and kernel function $k(x,x')=  \varphi^T(x) \varphi(x') \in \Re$. Define $\Phi_t := [\varphi(x_1), \ldots, \varphi(x_t)]^T \in \Re^{t\times M}$ and we get the $t\times t$ kernel matrix $K_t= \Phi_t \Phi_t^T$ and  $k_t(x) = \Phi_t \varphi(x) \in \Re^t$.
	
	In our weighted case, we have the weighted feature matrix $\tilde \Phi_t := W \Phi_t$, weighted kernel vector $\tilde k_t(x) = \tilde \Phi_t \varphi(x)$ and weighted kernel matrix $\tilde K_t= \tilde \Phi_t \tilde \Phi_t^T $.  
	Additionally,  the double-weighted feature matrix $\bar \Phi_t := W^2 \Phi_t$, and double-weighted kernel matrix $\bar K_t= \bar \Phi_t \bar \Phi_t^T $. Besides, we further define weighted Gram matrix $V_t = \sum_{s=1}^t w_s \varphi(x_s) \varphi(x_s)^T + \lambda_t I_\mathcal{H}\in  \Re^{M\times M}$ and double-weighted Gram matrix $\tilde V_t = \sum_{s=1}^t w_s^2 \varphi(x_s) \varphi(x_s)^T + \alpha_t I_\mathcal{H}  \in \Re^{M\times M}$. The full list of notations is deferred to Appendix \ref{sec: notations}.
	
	This explicit feature map enables us to directly establish the following result, which states that our weighted GP bandit generalizes the weighted linear bandit. The details and proof are stated in Appendix \ref{sec: equivalent} Lemma \ref{lem: equivalent}, where $\hat{\theta}_t = V_t^{-1} \sum_{s=1}^t w_s \varphi(x_s) y_s$. 
	\begin{remark}
		The weighted linear bandits in \cite[Equation 3]{russac2019weighted} can be recovered by taking $\tilde \mu_t(x) =\varphi(x)^T \hat{\theta}_t $ and $\varphi(x)=x$.
	\end{remark}
	In the following, based on this explicit feature space, we will establish confidence bounds for our weighted GP bandit under both stationary and non-stationary environments.

	\textbf{Confidence bound under stationary environments. \space}
	First we consider the stationary environment, where the reward function $f_t=f^*$ does not change with respect to time $t$. The following result shows how the posterior mean $\tilde \mu_t(x)$ is concentrated around the unknown reward function $f^*(x)$.
	\begin{theorem} \label{thm: confidence bound stationary}
		Let $f^*: D \rightarrow \Re $ be a member of the RKHS of real-valued functions on $D$ specified by kernel k, with RKHS norm bounded by $\|f^*\|_H \leq B$ and $\acute \sigma^2_t(x) = \lambda \|\varphi(x)\|^2_{V_t^{-1} \tilde V_t V_t^{-1}}$. Then, with probability at least $1-\delta$, the following concentration inequality holds:
		\begin{align*}
		|f^*(x) - \tilde \mu_t(x)| &\leq  \acute \sigma_t(x) B+  \frac{\acute \sigma_t(x)}{\sqrt{\lambda}}  R \sqrt{  2\log(\frac{1}{\delta}) +  2 \bar \gamma_t} \\
		&= \acute \sigma_t(x) \beta_t
		\end{align*}
	\end{theorem}

	\textbf{Proof Sketch for Theorem \ref{thm: confidence bound stationary}. } Following similar steps in \cite[Section 3.2]{abbasi2013online}, we first develop a self-normalized concentration bound on the weighted error sum $S_t= \sum_{s=1}^t w_s \varphi(x_s) \epsilon_s$. Then we bound $\|S_t\|_{\tilde V^{-1}_t}$ through double weighted information gain $\bar \gamma_t$. Finally we decompose $ |f^*(x) - \tilde \mu_t(x)|$ into two terms $\|\varphi(x)\|_{V_t^{-1} \tilde V_t V_t^{-1}}$ and  $(\| S_t \|_{ \tilde V_t^{-1}}+ \lambda_t \|  \theta^* \|_{ \tilde V_t^{-1} } )$, and then bound them separately. The formal proofs and auxiliary lemmas are deferred to Appendix \ref{sec: confidence bound stationary}. 
	
	We note that $\tilde \mu_t(x)$ here can be calculated by Equation \eqref{equ:mean}. As $\varphi(x)$ is involved in $V_t$ and $\tilde V_t$, we need to know the feature map $\varphi(x)$ before calculating $\acute \sigma_t(x)$, which is usually not practical. We resolve this issue in the following subsection by defining another predictive variance $\tilde \sigma_t(x)$.  
	
	With this confidence bound, we can claim that the standard kernelized bandit is only a special case of our weighted kernelized bandit. We defer the detailed explanation to Appendix \ref{sec: proof remark stationary } via Lemma \ref{lemma: hat sigma}. 
	
	\begin{remark} \label{rmk: stationary}
		The standard stationary case (IGP-UCB algorithm) \citep[Theorem 2]{chowdhury2017kernelized} is recovered by taking $\lambda=1$ and $w_t=1$.
	\end{remark}

	\textbf{Confidence bounds for non-stationary cases. \space} In the non-stationary case, it is not guaranteed that the actual reward function  $f_t(x_t)$ always lies inside of confidence ellipsoid in Theorem \ref{thm: confidence bound stationary} because of the time variations of environments. As  did in weighted linear bandits \citep{russac2019weighted}, we introduce a surrogate parameter. $
	m_t(x)=  \varphi(x)^T V^{-1}_{t-1} [\sum_{s=1}^{t-1} w_s \varphi(x_s)  f_s(x_s) + \lambda w_{t-1} \theta^*_t  ],\; \text{where}\; f_t^*(x)=\varphi(x)^T \theta^*_t
	$.
	
	We note that this surrogate parameter $m_t(x)$ \textit{is only used in the analysis of dynamic regret bound, and it is not involved in the implementation of our Algorithm \ref{alg:weighted}}. 
	
	
	Leveraging this surrogate parameter $m_t(x)$, we can show that the new posterior mean is still close to the true function in the non-stationary environment.
	i.e., it satisfies
	$|m_t(x) - \tilde \mu_{t-1}(x)| \leq \tilde \sigma_{t-1}(x) \beta_{t-1}$ where $\tilde \sigma^2_t(x)$ is defined in Equation \eqref{equ:variance}.

	\begin{theorem} \label{thm: confidence bound nonstationary}
		Let $\mathcal C_t = \{f_t : |f_t(x) - \tilde \mu_{t-1}(x)| \leq \tilde \sigma_{t-1}(x) \beta_{t-1}, \forall x \in D  \}$ denote the confidence ellipsoid. Then, $\forall \delta>0$,
		$\mathcal P (m_t \in \mathcal C_t) \geq 1-\delta$.
	\end{theorem}
	
	
	We remark that we cannot  directly generalize weighted linear bandit \citep{russac2019weighted} to nonlinear bandit by simply replacing $A_s$ in \cite{russac2019weighted} with feature map $\varphi(x_s)$. This is because we can explicitly calculate the weighted gram $V_t=\sum_{s=1}^t \omega_s A_s A_s^T + \lambda_t I_d$ in linear case, while in the nonlinear case  the weighted gram $V_t=\sum_{s=1}^t \omega_s \varphi(x_s) \varphi(x_s)^T  + \lambda_t I_H$ cannot be explicitly calculated since the feature map  $\varphi(x)$ is unknown. Therefore, calculating $\acute \sigma_t(x)^2 = \lambda ||\varphi(x)||^2_{V_t^{-1} \tilde V_t V_t^{-1} } $ is not practical. We overcome this by designing $\tilde \sigma_t(x)$ in Equation \eqref{equ:variance} (can be calculated without $\varphi(x)$) and $\tilde \sigma_t(x)$ plays the similar role as $\acute \sigma_t(x)$ in the confidence bound. The full proof is stated in Appendix \ref{sec: confidence bound nonstationary}.

	\section{Dynamic Regret}
	In this section, we aim to resolve the challenge \textbf{(II)} listed at the end of Section~\ref{sec:WGP} and obtain a sublinear regret bound for WGP-UCB (Algorithm \ref{alg:weighted}). In particular, we resort to Quadrature Fourier Features (QFF) approximation to find an upper bound over the sum of predictive variance, which allows us to  explicitly state the regret bound and analyze the order of regret bound. We further consider exponentially increasing weights of the form $w_t = \eta^{-t}$ to simplify the analysis, where $ 0<\eta<1$ is the discounting factor.
	


	\textbf{Quadrature Fourier Features (QFF) approximation.} In some previous work \citep{abbasi2011improved, russac2019weighted}, the feature dimension explicitly appears in the regret bound, which makes regret bound become trivial if the feature space is of infinite dimension. To overcome this, we find an approximate feature map $\breve \varphi$, such that the error of approximation is controlled in the infinite-dimensional feature space. 
	
	
	We consider a finite-dimension feature map $\breve{\varphi}(\cdot): D \rightarrow \Re^m$  such that 
	it has a uniform approximation guarantee \citep{mutny2019efficient}, i.e., for any $x,y \in D$, $\sup_{x,y}|k(x,y) - \breve{\varphi}(x)^T \breve{\varphi}(y)| \le \varepsilon_m$. 
	If $D=[0,1]^d$, for common kernels such as the Squared Exponential or the modified Matern kernel, we construct the feature map where $\bar m \in \Na$ and $m=\bar m^d$,
	\begin{align*}
	\breve \varphi(x)_i = \left\{
	\begin{aligned}
	&\sqrt{v(\rho_i)} \cos\big(\frac{\sqrt{2}}{l} \rho_i^T x\big), \text{if } 1 \leq i \leq m \\
	&\sqrt{v(\rho_{i-m})} \sin\big(\frac{\sqrt{2}}{l} \rho_{i-m}^T x\big), \text{if } m + 1 \leq i \leq 2m 
	\end{aligned}
	\right.
	\end{align*}
	where $v(\rho)=\prod_{j=1}^d \frac{2^{\bar m -1} \bar m \,!}{\bar m H_{\bar m - 1}(\rho_j)^2}$ and $H_i$ is the $i$th Hermite polynomial \citep{hildebrand1987introduction}. The set $\{\rho_1, \ldots, \rho_j \}= P_{\bar m} \times \ldots \times P_{\bar m} $ ($d$ times) where $P_{\bar m}$ is the set of $\bar m$ roots of  the $i$th Hermite polynomial $H_i$.

	We then define $\breve \Phi_t = W [\breve \varphi(x_1), \ldots, \breve  \varphi(x_t)]^T$, $\breve k_t(x) = \breve \Phi_t \breve \varphi(x)$, $\breve K_t= \breve \Phi_t \breve \Phi_t^T$, $\breve V_t= \breve \Phi_t^T \breve \Phi_t+ \lambda_t I_\mathcal{H}$,  $ \breve{\sigma}^2_t(x) =  k(x,x) -  \breve  k_t(x)^T (\breve K_t + \lambda_t I_t )^{-1} \breve  k_t(x)= \lambda_t \|\breve \varphi(x)\|^2_{\breve V_t^{-1}} $, and $\breve \gamma_t =   \frac{1}{2} \log \det (I + \lambda_t^{-1} \breve \Phi_t \breve \Phi_t^T  )$. For SE kernel, QFF error is bounded  by $\varepsilon_m = O(\frac{d 2^{d-1}}{(\bar m l^2)^{\bar m}})$  \citep[Lemma 14]{chowdhury2019bayesian}.
	
	\textbf{Bounding the sum of predictive variance. \space} Leveraging the QFF and the associated error bound, we can achieve a novel weight-dependent upper bound for the sum of predictive variance.
	$$ 
	\sum_{t=1}^T \tilde{\sigma}_{t-1}(x_t) \leq   \sqrt{4\lambda T \breve \gamma_T + 2\lambda m T^2 \log(1/\eta)} + \frac{  T \sqrt{\epsilon_m}}{1-\eta}.
	$$
	i.e., we approximate it with some finite dimension results and we can show that the approximation error part is small, through $\frac{\beta_T T \sqrt{\varepsilon_m}}{1-\eta}=O(1)$ after properly tuning $\eta$ and $m$. The detailed proof is in Appendix \ref{sec: approximation error} and \ref{sec: hat sigma}.
	
	We have tried to simply extend standard results in \cite{chowdhury2017kernelized}, however we found that we cannot bound $\sum_{t=1}^T \tilde \sigma_{t-1}(x_t)$ with weighted MIG $\tilde \gamma_t$, i.e.,  $\sum_{t=1}^T \tilde \sigma_{t-1}(x_t) \leq  \sqrt{\lambda T \log \det(I+\lambda^{-1} \tilde K_T )}$, which cannot be bounded through $\tilde \gamma_t =\max_{A\subset D: |A|=t} \frac{1}{2} \log \; \det(I + \lambda_t^{-1} \tilde K_t )$ because $\lambda_t = \lambda \eta^{-t} > \lambda$. We overcome it by truncating the feature space via QFF and bound the finite part with a samll approximation error, as shown in  Appendix \ref{sec: hat sigma}.

	
	
	
	\textbf{Regret bound of WGP-UCB with QFF approximation.}
	With the novel upper bound above, we can state the dynamic regret bound of WGP-UCB with QFF approximation.
	
	\begin{theorem} \label{thm: regret bound}
		Let $f_t \in H_k(D)$, $\|f_t\|_H \leq B$ and $k(x,x)\leq 1$. Then, with probability at least $1-\delta$, the dynamic regret $R_T$ is bounded by
		\begin{align*}
		O \Big( \beta_T &\sqrt{T \breve \gamma_T +  m T^2 \log(\frac{1}{\eta})}
		+ c^{\frac{3}{2}} B_T \sqrt{\breve \gamma_t +  m c \log(\frac{1}{\eta})}\\ 
		&+\frac{B\eta^c}{1-\eta}T    
		+ B_T \frac{c^2 \sqrt{\varepsilon_m}}{1-\eta} + \frac{\beta_T T \sqrt{\varepsilon_m}}{1-\eta} \Big)
		\end{align*}
		where $c\geq 1$ is an integer, $0 <\eta <1$, and $ \beta_t =   B+  \frac{1}{\sqrt{\lambda}}  R \sqrt{  2\log(\frac{1}{\delta}) +  2 \bar \gamma_t}$. 
	\end{theorem}

	\textbf{Proof Sketch for Theorem \ref{thm: regret bound}.} There are mainly three steps in this proof. First, we separate the stationary and non-stationary parts in the instantaneous regret $r_t$. They are bounded by $2 \beta_{t-1} \tilde \sigma_{t-1}(x_t)$ and $2 \sum_{p=t-c}^{t-1}\|f_p- f_{p+1}\|_H  \frac{\eta^{-1/2}}{\lambda} \sum_{s=t-c}^p \tilde \sigma_{s-1}(x_s)   + \frac{4B\eta^c}{\lambda (1-\eta)}$, respectively. As pointed out by \cite[p.4]{zhao2021non}, the statement $\lambda_{\max}(V_{t-1}^{-1} \sum_{s=t-D}^p \eta^{-s} A_s A_s^T)\leq 1$ in  \cite[p.18]{russac2019weighted} is not true. We fix this error in our proof as well, which introduces extra term $\sum_{s=t-c}^p \tilde \sigma_{s-1}(x_s)$. Secondly, we leverage the new bound for $\sum_{t=1}^T \tilde{\sigma}_{t-1}(x_t)$ developed above, which is $ \sqrt{4\lambda T \breve \gamma_T + 2\lambda m T^2 \log(1/\eta)} + \frac{ T \sqrt{\varepsilon_m}}{1-\eta}$. Finally, we bound  $\eta^{-1/2} \sum_{s=t-c}^t \tilde \sigma_{s-1}(x_s)$ through $\sqrt{4\lambda c \breve \gamma_t + 2\lambda m c^2 \log(1/\eta)} + \frac{c \sqrt{\varepsilon_m}}{1-\eta}$ with QFF, which is composed of finite approximation result and associated error. The full proof is in Appendix \ref{sec: hat sigma}-\ref{sec: proof of regret bound}.
	
	\textbf{Order analysis of regret bound.}
	We start analysing the order of regret bound by define $\dot \gamma_T = \max\{\bar \gamma_T, \breve \gamma_T\}$. It is the maximum between double-weighted MIG and weighted MIG with QFF approximation, which is called \emph{combined weighted MIG}. By optimally setting $c=\frac{\log T}{1-\eta}$ and $\bar m = \log_{4/e} (T^3 \dot \gamma_T^{3/2})$, we have the order analysis as follows. The detail is deferred to Appendix \ref{sec: proof of order analysis}.
	\begin{corollary} \label{cor: order}
		If $B_T$ is known, the dynamic regret bound is $\tilde O(\dot \gamma_T^{7/8} B_T^{1/4} T^{3/4})$ by optimally choosing $\eta = 1 - \dot \gamma_T^{-1/4}B_T^{1/2}T^{-1/2}$. If  $B_T$ is unknown, the dynamic regret bound is $\tilde O(\dot \gamma_T^{7/8} B_T T^{3/4})$ by optimally choosing $\eta = 1 - \dot \gamma_T^{-1/4}T^{-1/2}$.
	\end{corollary}
	
	\begin{remark}
		This regret bound achieves the same order as \cite{zhou2021no} where restarting and sliding window mechanisms are used. It is also a generalization of \cite{zhao2021non}, which studied non-stationary linear bandit and fixed the error of largest eigenvalue in previous papers \citep{cheung2019learning, russac2019weighted,zhao2020simple}.
	\end{remark}

	\section{Upper bounds on  Maximum Information Gain}
	In this section, we aim to resolve the challenge \textbf{(III)} mentioned in Section~\ref{sec:WGP}, i.e., finding an explicit upper bound on MIG. In our case, we have multiple weighted MIGs and hence standard results fail. To resolve this issue, we generalize the idea in~\cite{vakili2021information} to our weighted case by exploiting the tail properties in the feature maps given by Mercer's theorem.
	
	

	In particular, our bounds on MIGs are based on a finite dimensional projection of the kernel, we start with outlining the details of this projection.  For each element in $K_t$, we recall Equation \eqref{equ: kernel} by Mercer's Theorem, where $c_m \in \Re^+$ and $\phi_m \in \mathcal{H}_k$ for $m \geq 1$. $\{c_m\}_{m=1}^\infty$ is the eigenvalue sequence in decreasing order. $\{\phi_m\}_{m=1}^\infty$ are the eigenfeature of $k$. Similarly, for each element in double weighted kernel matrix $\bar K_t$, we have double weighted kernel function 
	$\bar k(x_i,x_j)= w_i w_j k(x_i,x_j) = \sum_{m=1}^\infty w_i w_j c_m \phi_m(x_i) \phi_m(x_j)$.
	
	\begin{assumption}
		(1) $\forall x, x' \in D, |k(x,x')|\leq \dot k$, for some $\dot k>0$  (2) $\forall m\in \Na, \forall x\in D, |\phi_m(x)|\leq \psi$, for some $\psi>0$. 
	\end{assumption}
	
	In particular, we consider a N-dimensional projection \citep{vakili2021information}, where the $N$-dimensional feature space is $\Psi_N=[\phi_1(x),\phi_2(x), \ldots, \phi_N(x)]^T$. We keep the first $N$-dimension feature in kernel $\bar k_P(x_i,x_j)=w_i w_j \sum_{m=1}^N c_m \phi_m(x_i) \phi_m(x_j)$. The remaining part is $\bar k_O(x,x')=\bar k(x,x')-\bar k_P(x,x')$.
	
	We define the following quantity based on the tail mass of the eigenvalues of $m$, $\delta_N=  \sum_{m=N+1}^\infty c_m \psi^2$.
	Then for all $x,x' \in D$, we have $ k_O(x,x')\leq \delta_N$. For some kernel $k$, if $c_m$ diminishes at a sufficiently fast rate, then $ \delta_N$ becomes arbitrarily small when $N$ is large enough, which will be discussed in Corollary \ref{cor: gamma bound eigendecay}.
	
	\begin{figure*}
		\centering
		\begin{subfigure}{.3\textwidth}
			\centering
			\includegraphics[width=6cm, height=4cm]{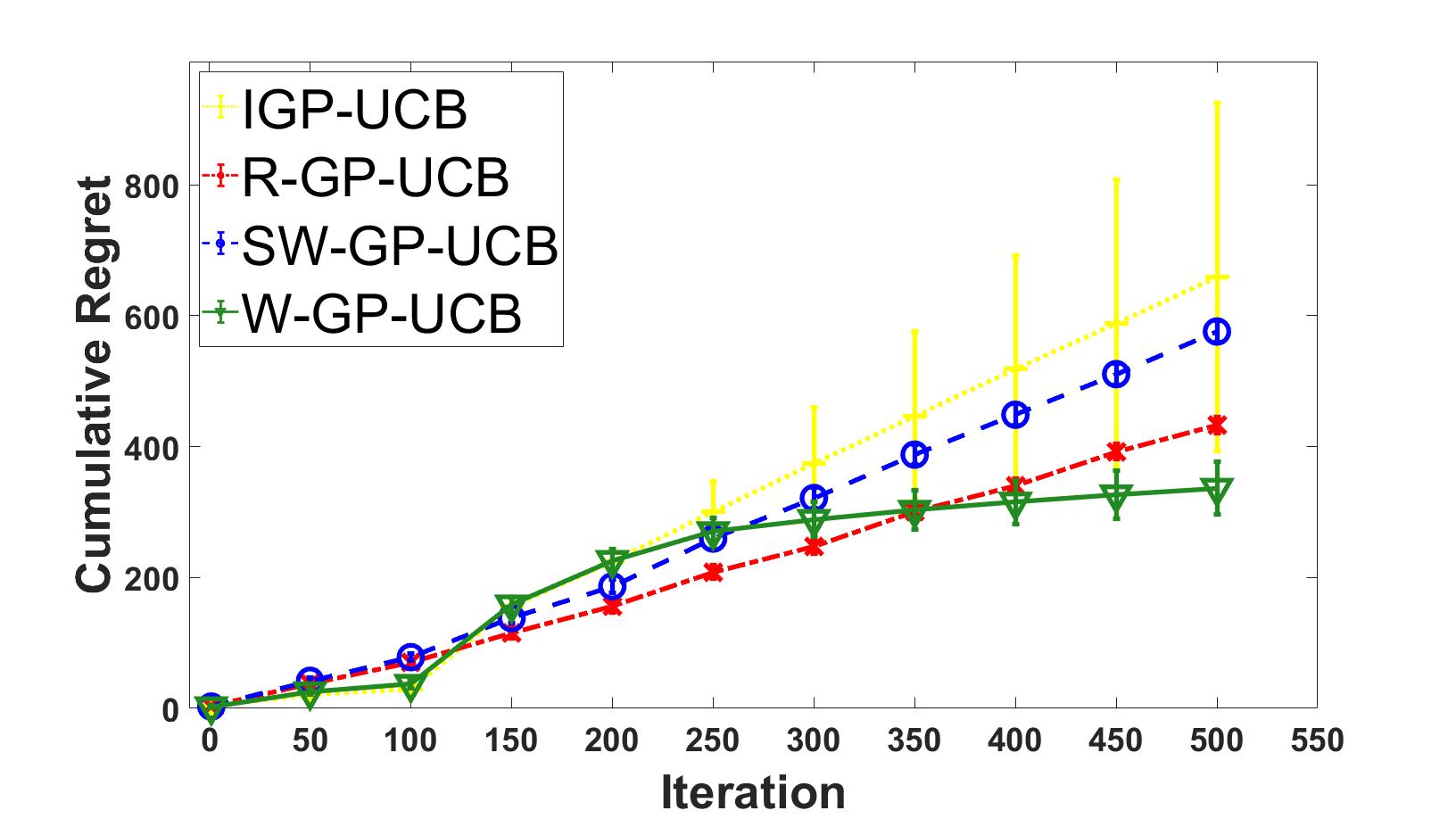}
			\caption{Abruptly-change, SE kernel}
		\end{subfigure}%
		\begin{subfigure}{.3\textwidth}
			\centering
			\includegraphics[width=6cm, height=4cm]{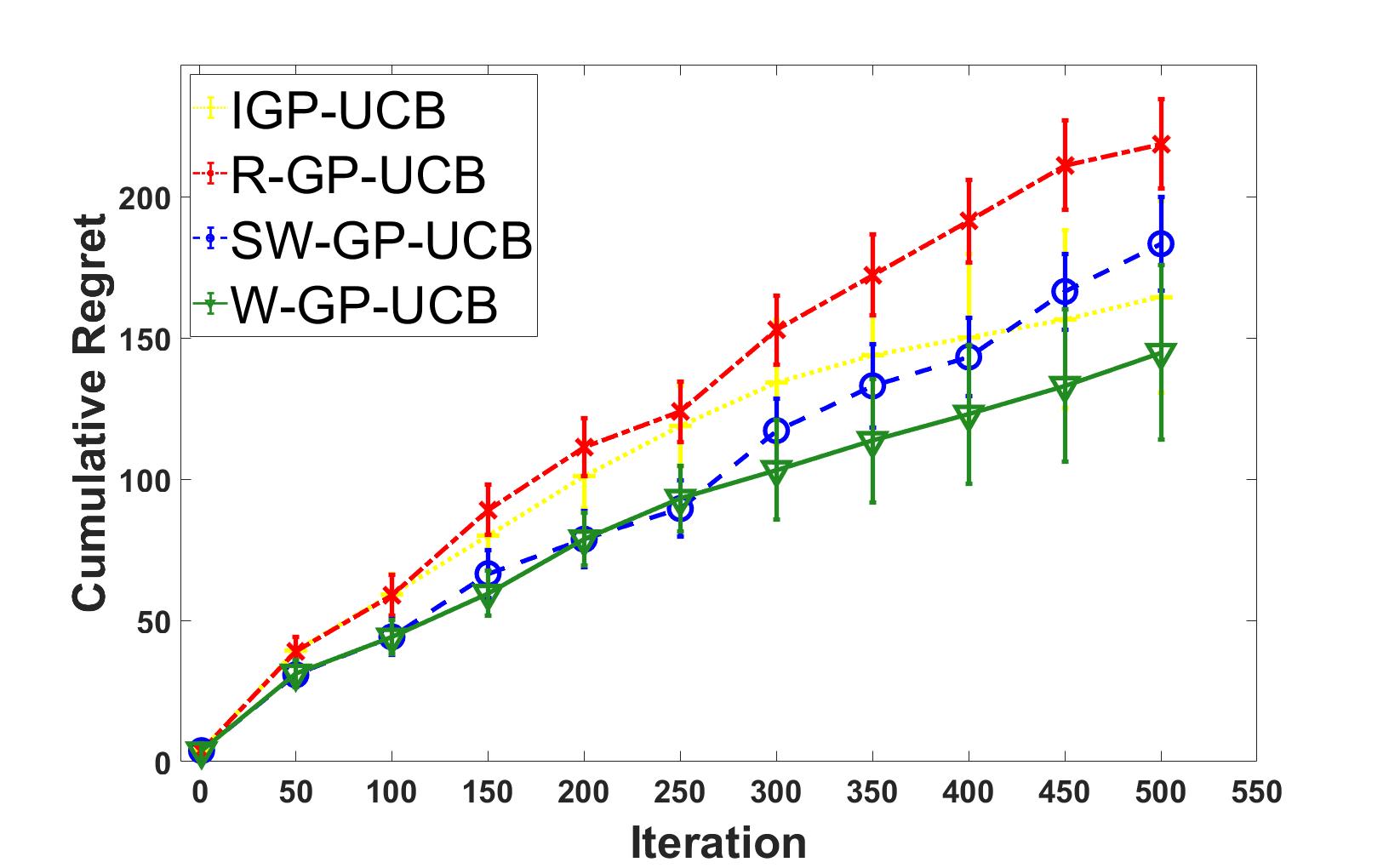}
			\caption{Slowly-change, SE kernel}
		\end{subfigure}%
		\begin{subfigure}{.3\textwidth}
			\centering
			\includegraphics[width=6cm, height=4cm]{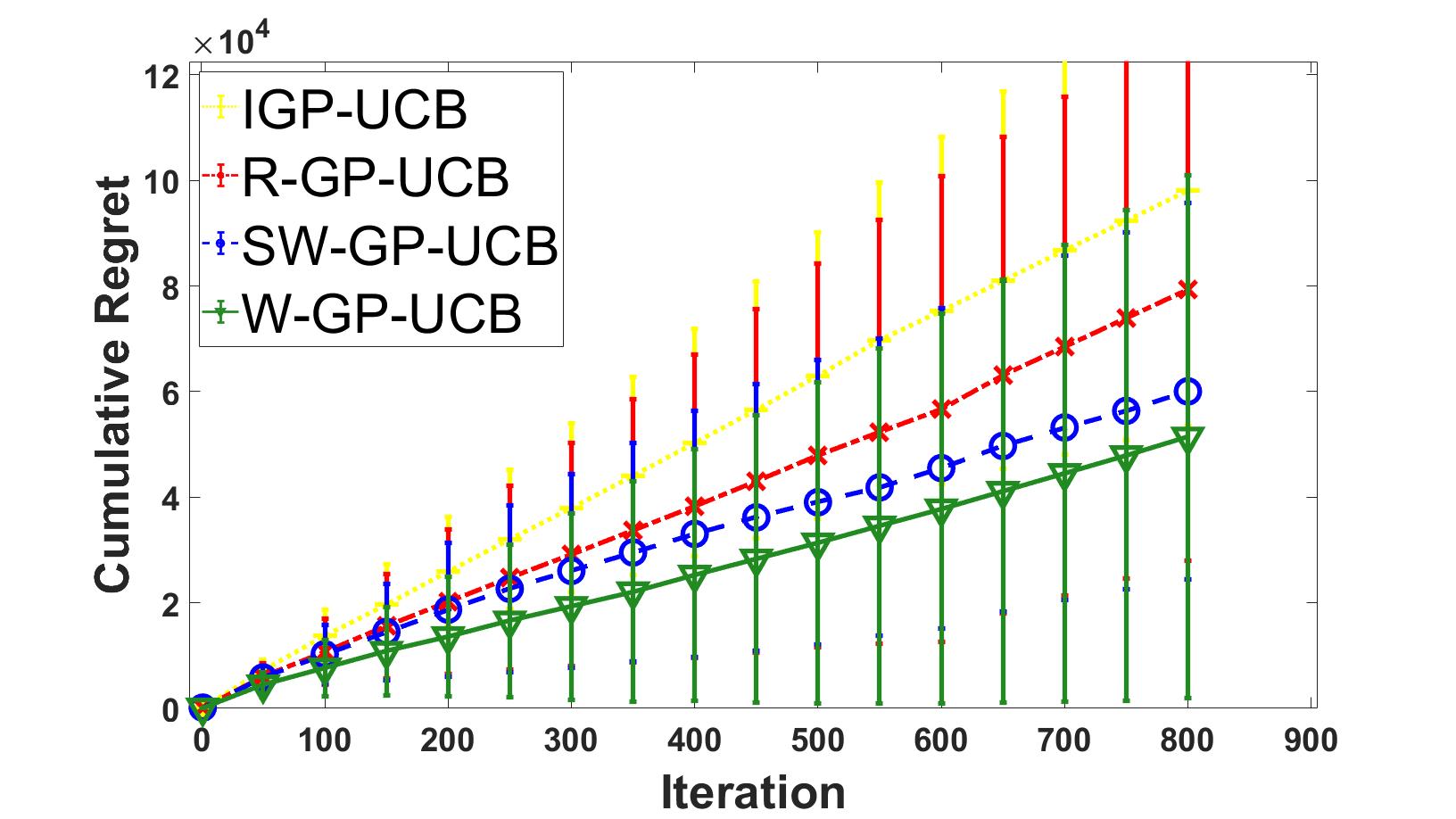}
			\caption{Stock market data}
		\end{subfigure}
		\caption{Average cumulative regret of four algorithms in three different scenarios }
		\label{fig: simulation}
	\end{figure*}

	\textbf{Universal Bound: \space} Based on this eigendecay, we provide a universal upper bound for both $\bar \gamma_T$ and $\breve \gamma_T$, which states that the order $\tilde O(\log(T))$ holds for combined weighted MIG $\dot \gamma_T$ with any increasing weights $\{w_s\}_{s=1}^t$. The full proof is stated in Appendix \ref{sec: universal bound}.

	\begin{theorem} \label{thm: gamma bound}
		If Assumption 1 holds, $\dot \gamma_T = \max\{\bar \gamma_T, \breve \gamma_T\}\leq \frac{N}{2} \log \Big(1+ \frac{\dot k  T}{\lambda N } \Big) + \frac{T}{2\lambda } \delta_N$ 
	\end{theorem}
	
	The expression in Theorem \ref{thm: gamma bound} can be predigested as $\dot \gamma_T = O(N \log(T) + \delta_N T)$, which resolves challenge \textbf{(III)} mentioned in Section~\ref{sec:WGP}. To be more specific, the following  remark provides an explicit form of the upper bound for SE kernel, which has a exponential eigendecay \citep{belkin2018approximation, vakili2021information}.

	\begin{remark}
		For SE kernel, we have $c_m=O(exp(-m^{1/d}))$ and $\dot \gamma_T = O\big(\log^{d+1} (T) \big)$.
	\end{remark}

	\textbf{Weight-dependent bound.} Specifically, if the weights are exponentially increasing, we achieve a tighter upper bound for double weighted MIG $\bar \gamma_T$ and single weighted MIG with QFF $\breve \gamma_T$, respectively. This novel upper bound depends on the discount factor $\eta$ and holds under any time horizon $T$.

	\begin{theorem} \label{thm: bar gamma bound}
		If Assumption 1 holds and weight $\omega_t = \eta^{-t}$, the following upper bound on $\bar \gamma_T$ 
		holds for all $N \in \Na$.
		\begin{align*}
		\bar  \gamma_T  &\le \frac{N}{2} \log \Big(1+ \frac{\dot k  }{\lambda N (1-\eta^2)} \Big) + \frac{1}{2\lambda (1-\eta^2) } \delta_N.
		\end{align*}
		
		
	\end{theorem}
	If the polynomial or exponential conditions on the the eigendecay of $k$ are provided, a tighter bound is established. 
	
	\begin{corollary} \label{cor: gamma bound eigendecay}
		1. Under the $(C_p, \beta_p)$ polynomial eigendecay condition, i.e., $c_m \leq C_p m^{-\beta_p}$,
		\begin{align*}
		\bar \gamma_T \le \Big( \big( \frac{C_p \psi^2}{\lambda(1-\eta^2)}\big)^{\frac{1}{\beta_p}}\log^{-\frac{1}{\beta_p}} ( 1+ \frac{\dot k}{\lambda (1- \eta^2)})+ 1 \Big) \\
		\log( 1+ \frac{\dot k}{\lambda (1- \eta^2)}).
		\end{align*}
		2. Under the $(C_{e,1}, C_{e,2}, \beta_e = 1)$ exponential eigendecay condition, i.e. $c_m \leq C_{e,1} \text{\space exp}(-C_{e,2} m^{\beta_e})$,
		\begin{align*}
		\bar \gamma_T &\le \Big( \frac{1}{C_{e,2}} \big( \log (\frac{1}{1-\eta^2}) + C_{\beta_e} \big) +1 \Big)
		\log( 1+ \frac{\dot k}{\lambda (1- \eta^2)}) \\ &\text{where}\ C_{\beta_e} = \log\big( \frac{C_{e,1}\psi^2}{\lambda C_{e,2}} \big).
		\end{align*}
		
	\end{corollary}
	Similar results hold for $\breve \gamma_T$ except that $\frac{1}{1-\eta}$ is replaced by  $\frac{1}{1-\eta^2}$. For kernel with polynomial eigendecay condition, $\bar \gamma_T$  will play a role in the overall dynamic regret bound due to its leading term of $\big(\frac{1}{1-\eta^2}\big)^{1/\beta_p}$. However, for kernels with exponential eigendecay condition, it will not affect the overall dynamic regret bound since it only has the logarithmic dependency on $\frac{1}{1-\eta^2}$.

	\section{Experiments}
	We numerically compare the performance of IGP-UCB  \citep{chowdhury2017kernelized},  R-GP-UCB \citep{zhou2021no}, SW-GP-UCB\citep{zhou2021no}, WGP-UCB (Algorithm \ref{alg:weighted}) on both synthetic and real-world data.  The restarting period $H$, sliding window $SW$ and exponential weight $\eta$ are set order-wise by theory (Corollary \ref{cor: order}, Remark 1 \citep{zhou2021no}).

	\textbf{Synthetic data. } We develop experiments on both abruptly-changing environments and the slowly-varying environments. We generate the objective function $f\in H_k(D)$ where $D$ is a discretization of $[0,1]$ into 100 evenly spaced points. We use SE kernel with $l=0.2$ as our kernel function $k(\cdot,x_i)$ where supporting points $x_i \in D$. The reward function is generated as $f(\cdot)=\sum_{i=1}^M \alpha_i k(\cdot, x_i)$ with $\alpha_i \in [-1,1]$ uniformly sampled and $M=100$. In the first experiment (Figure \ref{fig: simulation} (a)), we observe the empirical performance of all algorithms in an abruptly changing environment. The reward function changes at 2 points, i.e., before $t=100$, $f_t=f_1^*$; for $t\in[100,200]$, $f_t=f_2^*$; for $t\in[200,500]$, $f_t=f_3^*$. The second experiment corresponds to a slowly-changing environment (Figure \ref{fig: simulation} (b)), where when $t<=T/2$, $f_t= f_4^* + (f_5^*-f_4^*) 2t/T$, and when $t>T/2$, $f_t= f_5^* + (f_6^*-f_5^*) (2t-T)/T$. All $f_i^*$'s are randomly sampled within RKHS and the cumulative regret is averaged on 100 independent experiments with error bars in the figure. 


	\textbf{Stock market data. }We take the adjusted closing price of 29 stocks for 823 days.\footnote{https://www.quandl.com/data/EOD-End-of-Day-US-Stock-Prices} We use the daily closing price as our time-varying reward function $f_t$ and  the empirical covariance of the stock price as our kernel function $k$.  We assume that investors would like to buy one stock upon opening and sell it right before closing, i.e., they want to  get much profit as possible after selling it on the same day. The regret is non-sublinear as the rewards in this dataset are heavy-tailed.
	
	\textbf{Observations. } We find that WGP-UCB outperforms three algorithms over all experiments.  Moreover, R-GP-UCB and SW-GP-UCB  completely drop outdated information and may not have enough information to make predictions. However, W-GP-UCB, retains outdated information through gradual discounts.
	
	\section{Conclusion}
	In this paper, we develop a framework for regret analysis under weighted Gaussian process regression by overcoming three critical challenges. We propose WGP-UCB algorithm for non-stationary bandit optimization and establish the first regret bound for weighted penalty algorithm in GP bandits. Our future direction is to improve the regret bound when time-varying budget $B_T$ is unknown. It would be interesting to adopt adaptive weights based on non-stationarity detection,  via maintaining different instances of algorithms with different starting times \citep{wei2021non}.
	
	\section{Acknowledgements}
	This work has been supported in part by NSF grants: 2112471 (also partly funded by DHS), IIS-2007055, CNS-2106933, and CNS-1901057 and a grant from the Army Research Office: W911NF-21-1-0244. We thank all reviewers for their comments and suggestions.
	
	
	
	

	
	\bibliography{sample}

\begin{thebibliography}{}

\bibitem[Abbasi-Yadkori, 2013]{abbasi2013online}
Abbasi-Yadkori, Y. (2013).
\newblock Online learning for linearly parametrized control problems.

\bibitem[Abbasi-Yadkori et~al., 2011]{abbasi2011improved}
Abbasi-Yadkori, Y., P{\'a}l, D., and Szepesv{\'a}ri, C. (2011).
\newblock Improved algorithms for linear stochastic bandits.
\newblock In {\em NIPS}, volume~11, pages 2312--2320.

\bibitem[Auer, 2002]{auer2002using}
Auer, P. (2002).
\newblock Using confidence bounds for exploitation-exploration trade-offs.
\newblock {\em Journal of Machine Learning Research}, 3(Nov):397--422.

\bibitem[Auer et~al., 2019]{auer2019adaptively}
Auer, P., Gajane, P., and Ortner, R. (2019).
\newblock Adaptively tracking the best bandit arm with an unknown number of
  distribution changes.
\newblock In {\em Conference on Learning Theory}, pages 138--158. PMLR.

\bibitem[Belkin, 2018]{belkin2018approximation}
Belkin, M. (2018).
\newblock Approximation beats concentration? an approximation view on inference
  with smooth radial kernels.
\newblock In {\em Conference On Learning Theory}, pages 1348--1361. PMLR.

\bibitem[Besbes et~al., 2014]{besbes2014stochastic}
Besbes, O., Gur, Y., and Zeevi, A. (2014).
\newblock Stochastic multi-armed-bandit problem with non-stationary rewards.
\newblock {\em Advances in neural information processing systems}, 27:199--207.

\bibitem[Besbes et~al., 2015]{besbes2015non}
Besbes, O., Gur, Y., and Zeevi, A. (2015).
\newblock Non-stationary stochastic optimization.
\newblock {\em Operations research}, 63(5):1227--1244.

\bibitem[Besbes et~al., 2019]{besbes2019optimal}
Besbes, O., Gur, Y., and Zeevi, A. (2019).
\newblock Optimal exploration--exploitation in a multi-armed bandit problem
  with non-stationary rewards.
\newblock {\em Stochastic Systems}, 9(4):319--337.

\bibitem[Besson and Kaufmann, 2019]{besson2019generalized}
Besson, L. and Kaufmann, E. (2019).
\newblock The generalized likelihood ratio test meets klucb: an improved
  algorithm for piece-wise non-stationary bandits.
\newblock {\em arXiv preprint arXiv:1902.01575}.

\bibitem[Bogunovic et~al., 2016]{bogunovic2016time}
Bogunovic, I., Scarlett, J., and Cevher, V. (2016).
\newblock Time-varying gaussian process bandit optimization.
\newblock In {\em Artificial Intelligence and Statistics}, pages 314--323.
  PMLR.

\bibitem[Bubeck and Cesa-Bianchi, 2012]{bubeck2012regret}
Bubeck, S. and Cesa-Bianchi, N. (2012).
\newblock Regret analysis of stochastic and nonstochastic multi-armed bandit
  problems.
\newblock {\em arXiv preprint arXiv:1204.5721}.

\bibitem[Cao et~al., 2019]{cao2019nearly}
Cao, Y., Wen, Z., Kveton, B., and Xie, Y. (2019).
\newblock Nearly optimal adaptive procedure with change detection for
  piecewise-stationary bandit.
\newblock In {\em The 22nd International Conference on Artificial Intelligence
  and Statistics}, pages 418--427. PMLR.

\bibitem[Cheung et~al., 2019]{cheung2019learning}
Cheung, W.~C., Simchi-Levi, D., and Zhu, R. (2019).
\newblock Learning to optimize under non-stationarity.
\newblock In {\em The 22nd International Conference on Artificial Intelligence
  and Statistics}, pages 1079--1087. PMLR.

\bibitem[Chowdhury and Gopalan, 2017]{chowdhury2017kernelized}
Chowdhury, S.~R. and Gopalan, A. (2017).
\newblock On kernelized multi-armed bandits.
\newblock {\em arXiv preprint arXiv:1704.00445}.

\bibitem[Chowdhury and Gopalan, 2019]{chowdhury2019bayesian}
Chowdhury, S.~R. and Gopalan, A. (2019).
\newblock Bayesian optimization under heavy-tailed payoffs.
\newblock {\em arXiv preprint arXiv:1909.07040}.

\bibitem[Garivier and Moulines, 2011]{garivier2011upper}
Garivier, A. and Moulines, E. (2011).
\newblock On upper-confidence bound policies for switching bandit problems.
\newblock In {\em International Conference on Algorithmic Learning Theory},
  pages 174--188. Springer.

\bibitem[Hildebrand, 1987]{hildebrand1987introduction}
Hildebrand, F.~B. (1987).
\newblock {\em Introduction to numerical analysis}.
\newblock Courier Corporation.

\bibitem[Huo and Fu, 2017]{huo2017risk}
Huo, X. and Fu, F. (2017).
\newblock Risk-aware multi-armed bandit problem with application to portfolio
  selection.
\newblock {\em Royal Society open science}, 4(11):171377.

\bibitem[Kanagawa et~al., 2018]{kanagawa2018gaussian}
Kanagawa, M., Hennig, P., Sejdinovic, D., and Sriperumbudur, B.~K. (2018).
\newblock Gaussian processes and kernel methods: A review on connections and
  equivalences.
\newblock {\em arXiv preprint arXiv:1807.02582}.

\bibitem[Kim and Tewari, 2020]{kim2020randomized}
Kim, B. and Tewari, A. (2020).
\newblock Randomized exploration for non-stationary stochastic linear bandits.
\newblock In {\em Conference on Uncertainty in Artificial Intelligence}, pages
  71--80. PMLR.

\bibitem[Li et~al., 2010]{li2010contextual}
Li, L., Chu, W., Langford, J., and Schapire, R.~E. (2010).
\newblock A contextual-bandit approach to personalized news article
  recommendation.
\newblock In {\em Proceedings of the 19th international conference on World
  wide web}, pages 661--670.

\bibitem[Liu et~al., 2018]{liu2018change}
Liu, F., Lee, J., and Shroff, N. (2018).
\newblock A change-detection based framework for piecewise-stationary
  multi-armed bandit problem.
\newblock In {\em Proceedings of the AAAI Conference on Artificial
  Intelligence}, volume~32.

\bibitem[Micchelli et~al., 2006]{micchelli2006universal}
Micchelli, C.~A., Xu, Y., and Zhang, H. (2006).
\newblock Universal kernels.
\newblock {\em Journal of Machine Learning Research}, 7(12).

\bibitem[Mutn{\`y} and Krause, 2019]{mutny2019efficient}
Mutn{\`y}, M. and Krause, A. (2019).
\newblock Efficient high dimensional bayesian optimization with additivity and
  quadrature fourier features.
\newblock {\em Advances in Neural Information Processing Systems 31}, pages
  9005--9016.

\bibitem[Robbins, 1952]{robbins1952some}
Robbins, H. (1952).
\newblock Some aspects of the sequential design of experiments.
\newblock {\em Bulletin of the American Mathematical Society}, 58(5):527--535.

\bibitem[Russac et~al., 2020]{russac2020algorithms}
Russac, Y., Capp{\'e}, O., and Garivier, A. (2020).
\newblock Algorithms for non-stationary generalized linear bandits.
\newblock {\em arXiv preprint arXiv:2003.10113}.

\bibitem[Russac et~al., 2019]{russac2019weighted}
Russac, Y., Vernade, C., and Capp{\'e}, O. (2019).
\newblock Weighted linear bandits for non-stationary environments.
\newblock In {\em Advances in Neural Information Processing Systems}, pages
  12040--12049.

\bibitem[Srinivas et~al., 2009]{srinivas2009gaussian}
Srinivas, N., Krause, A., Kakade, S.~M., and Seeger, M. (2009).
\newblock Gaussian process optimization in the bandit setting: No regret and
  experimental design.
\newblock {\em arXiv preprint arXiv:0912.3995}.

\bibitem[Vakili et~al., 2021]{vakili2021information}
Vakili, S., Khezeli, K., and Picheny, V. (2021).
\newblock On information gain and regret bounds in gaussian process bandits.
\newblock In {\em International Conference on Artificial Intelligence and
  Statistics}, pages 82--90. PMLR.

\bibitem[Wei and Luo, 2021]{wei2021non}
Wei, C.-Y. and Luo, H. (2021).
\newblock Non-stationary reinforcement learning without prior knowledge: An
  optimal black-box approach.
\newblock {\em arXiv preprint arXiv:2102.05406}.

\bibitem[Zhao and Zhang, 2021]{zhao2021non}
Zhao, P. and Zhang, L. (2021).
\newblock Non-stationary linear bandits revisited.
\newblock {\em arXiv preprint arXiv:2103.05324}.

\bibitem[Zhao et~al., 2020]{zhao2020simple}
Zhao, P., Zhang, L., Jiang, Y., and Zhou, Z.-H. (2020).
\newblock A simple approach for non-stationary linear bandits.
\newblock In {\em International Conference on Artificial Intelligence and
  Statistics}, pages 746--755. PMLR.

\bibitem[Zhou et~al., 2021]{zhou2021nearly}
Zhou, D., Gu, Q., and Szepesvari, C. (2021).
\newblock Nearly minimax optimal reinforcement learning for linear mixture
  markov decision processes.
\newblock In {\em Conference on Learning Theory}, pages 4532--4576. PMLR.

\bibitem[Zhou and Shroff, 2021]{zhou2021no}
Zhou, X. and Shroff, N. (2021).
\newblock No-regret algorithms for time-varying bayesian optimization.
\newblock In {\em 2021 55th Annual Conference on Information Sciences and
  Systems (CISS)}, pages 1--6. IEEE.

\bibitem[Zhou et~al., 2019]{zhou2019non}
Zhou, Y., Shen, C., and van~der Schaar, M. (2019).
\newblock A non-stationary online learning approach to mobility management.
\newblock {\em IEEE Transactions on Wireless Communications}, 18(2):1434--1446.

\end{thebibliography}

	\onecolumn
	\aistatstitle{Weighted Gaussian Process Bandits for Non-stationary Environments: \\
		Supplementary Materials}
	
	\appendix
	\section*{Appendix}
	
	\section{List of notations} \label{sec: notations}
	In this section we provide the full list of notations.
	\begin{itemize}
		\item Regularization and weight : $\lambda_t = \lambda w_t, \alpha_t = \lambda w_t^2, w_t = \eta^{-t}$
		\item Weighted observations : $\tilde y_{1:t} = W y_{1:t}= [\sqrt{w_1} y_1, \ldots, \sqrt{w_t} y_t]^T$
		\item Weight matrix : $W=\text{diag}(\sqrt{w_1}, \sqrt{w_2}, \ldots, \sqrt{w_t})$
		\item Feature matrix : $\Phi_t = [\varphi(x_1), \ldots, \varphi(x_t)]^T$
		\item Weighted feature matrix : $\tilde \Phi_t = W \Phi_t, \bar \Phi_t = W^2 \Phi_t,\breve \Phi_t = W [\breve \varphi(x_1), \ldots, \breve  \varphi(x_t)]^T$
		\item Kernel vector : $k_t(x) = \Phi_t \varphi(x)$
		\item Weighted kernel vector : $\tilde k_t(x) = \tilde \Phi_t \varphi(x) = W \Phi_t \varphi(x), \breve k_t(x) = \breve \Phi_t \breve \varphi(x)$
		\item Kernel matrix : $K_t=\Phi_t \Phi_t^T$
		\item Weighted kernel matrix : $\tilde K_t=W \Phi_t \Phi_t^T W^T, \bar K_t=W^2 \Phi_t \Phi_t^T W^{2T}, \breve K_t= \breve \Phi_t \breve \Phi_t^T$
		\item Weighted error sum : $S_t= \sum_{s=1}^t w_s \varphi(x_s) \epsilon_s$
		\item Weighted Gram matrix : \\
		$V_t = \sum_{s=1}^t w_s \varphi(x_s) \varphi(x_s)^T + \lambda_t I_\mathcal{H} = \tilde \Phi_t^T \tilde \Phi_t+\lambda_t I_\mathcal{H}=  \Phi_t^T W^2  \Phi_t+\lambda_t I_\mathcal{H}$
		\item Weighted Gram matrix with QFF : $\breve V_t= \breve \Phi_t^T \breve \Phi_t+ \lambda_t I_\mathcal{H}$
		\item Double weighted Gram matrix : $\tilde V_t = \sum_{s=1}^t w_s^2 \varphi(x_s) \varphi(x_s)^T + \alpha_t I_\mathcal{H} = \Phi_t^T W^4  \Phi_t + \alpha_t I_\mathcal{H}$
		\item Predictive variance : $\tilde{\sigma}^2_t(x) =  k(x,x) -  \tilde  k_t(x)^T (\tilde K_t + \lambda_t I_t )^{-1} \tilde  k_t(x)= \lambda_t \|\varphi(x)\|^2_{V_t^{-1}}$
		\item Predictive variance with QFF : \\
		$\breve{\sigma}^2_t(x) = \breve k(x,x) -  \breve  k_t(x)^T (\breve K_t + \lambda_t I_t )^{-1} \breve  k_t(x)= \lambda_t \|\breve \varphi(x)\|^2_{\breve V_t^{-1}}$
		\item Loose predictive variance : $\acute \sigma^2_t(x) = \lambda \|\varphi(x)\|^2_{V_t^{-1} \tilde V_t V_t^{-1}}$
		\item Confidence bound, $\beta_t =   B+  \frac{1}{\sqrt{\lambda}}  R \sqrt{  2\log(\frac{1}{\delta}) +  2 \bar \gamma_t}$
		\item Weighted maximum information gain : $\tilde \gamma_t =  \max_{A\subset D: |A|=t} \frac{1}{2} \log  \det(I + \lambda_t^{-1} W K_A W^T )=\max_{A\subset D: |A|=t} \frac{1}{2} \log  \det (I + \lambda_t^{-1} W \Phi_t \Phi_t^T W^T )$
		\item Weighted maximum information gain with QFF : $\breve \gamma_t =   \frac{1}{2} \log  \det (I + \lambda_t^{-1} \breve \Phi_t \breve \Phi_t^T)$
		\item Double weighted maximum information gain : \\
		$\bar \gamma_t =  \max_{A\subset D: |A|=t} \frac{1}{2} \log \det(I + \alpha_t^{-1} W^2 K_A W^{2T} ) =$\\
		$\max_{A\subset D: |A|=t} \frac{1}{2} \log \det (I + \alpha_t^{-1} W^2 \Phi_t \Phi_t^T W^{2T} )$
		\item Combined weighted maximum information gain : $\dot \gamma_t = \max\{\bar \gamma_t, \breve \gamma_t\}$
	\end{itemize}

	\section{Proof of Confidence Bounds}
	
	\subsection{Connection with weighted linear bandits} \label{sec: equivalent}
	
	The following lemma states that the linear case in \cite{russac2019weighted} can be recovered by taking $\tilde \mu_t(x) =\varphi(x)^T \hat{\theta}_t $ and $\varphi(x)=x$ in WGP-UCB algorithm (Algorithm \ref{alg:weighted}).
	
	\begin{lemma} \label{lem: equivalent}
		Equation \eqref{equ:mean} is equivalent to $\varphi(x)^T \hat{\theta}_t$ where $\hat{\theta}_t = V_t^{-1} \sum_{s=1}^t w_s \varphi(x_s) y_s$ and $V_t= \sum_{s=1}^t w_s \varphi(x_s) \varphi(x_s)^T + \lambda_t I_H$.
	\end{lemma}
	\begin{proof}
		As $\hat{\theta}_t$ is the regularized weighted least-squares estimator of $\theta^*$ at time t in \cite{russac2019weighted}, we have
		\begin{align*}
		\tilde \mu_t(x) & = \varphi(x)^T \hat{\theta}_t  = \varphi(x)^T V_t^{-1}  \sum_{s=1}^t w_s \varphi(x_s) y_s \\
		& = \varphi(x)^T V_t^{-1}  \tilde \Phi_t^T \tilde y_{1:t} = \varphi(x)^T (\tilde \Phi_t^T \tilde \Phi_t + \lambda_t I_H)^{-1}  \tilde \Phi_t^T \tilde y_{1:t} \\
		& = \varphi(x)^T \tilde \Phi_t^T (\tilde \Phi_t \tilde \Phi_t^T + \lambda_t I_t)^{-1} \tilde y_{1:t}  = \tilde k_t(x)^T (\tilde K_t + \lambda_t I_t )^{-1} \tilde y_{1:t}.
		\end{align*}
		The second last equality holds by $V_t= \tilde \Phi_t^T \tilde \Phi_t + \lambda_t I_H $ and $(\tilde \Phi_t^T \tilde \Phi_t + \lambda_t I_H)^{-1} \tilde \Phi_t^T = \tilde \Phi_t^T (\tilde \Phi_t \tilde \Phi_t^T + \lambda_t I_t)^{-1}$.
	\end{proof}
	
	\subsection{Confidence Bounds for stationary environments} \label{sec: confidence bound stationary}
	In this section we present the detailed proof of confidence bounds for stationary environments.
	\subsubsection{Self-normalized Concentration}
	In the following lemma, we shows one concentration inequality about noise sequence $\epsilon_t$. We define weighted error sum as $S_t= \sum_{s=1}^t w_s \varphi(x_s) \epsilon_s  \in \Re^M$.
	
	\begin{lemma} \label{lem: bound_s_1}
		With probability at least $1-\delta$, the following holds simultaneously over all $t>0$:
		\begin{align}
		\|S_t\|_{\tilde V^{-1}_t} \leq R \sqrt{  2\log(\frac{1}{\delta}) + \log( \det(I_t + \alpha_t^{-1}  W^2 \Phi_t  \Phi_t^T W^{2T} )}
		\end{align}
	\end{lemma}
	
	\begin{proof}
		This result is adopted from  \cite[Section 3.2]{abbasi2013online}. 
		
		Similar to \cite[Equation 3.4]{abbasi2013online}, we have  $S_t=\sum_{s=1}^t w_s \varphi(x_s) \epsilon_s $ where  $m_k$ is replaced by $w_s \varphi(x_s)$ and $\epsilon_s$ is R-sub-Gaussian noise. Following \cite[Equation 3.5]{abbasi2013online}, this equation holds $\tilde V_t = \sum_{s=1}^t w_s^2 \varphi(x_s) \varphi(x_s)^T + \alpha_t I_\mathcal{H} = \Phi_t^T W^4  \Phi_t + \alpha_t I_\mathcal{H}$, where $V$ is replaced by $\alpha_t I_\mathcal{H}$ and $m_k$ is replaced by $w_s \varphi(x_s)$. Additionally, we can replace $M_{1:t}$ with $W^2 \Phi_t$.
		
		Following the analysis till \cite[Corollary 3.6]{abbasi2013online}, we have the following inequality by replacing $M_{1:t}$ and $V$ respectively,
		\begin{align*}
		\|S_t\|^2_{\tilde V^{-1}_t} \leq 2R^2  \log \Big(\frac{\det(I_t + W^2 \Phi_t (\alpha_t I_\mathcal{H})^{-1}  (W^2 \Phi_t)^T)^{1/2}}{\delta}\Big).
		\end{align*}
		We can get the final result by taking the square root on both sides of the above inequality.
	\end{proof}
	
	As $\bar \gamma_t = \max_{A\subset D: |A|=t} \frac{1}{2} \log \det (I + \alpha_t^{-1} W^2 \Phi_t \Phi_t^T W^{2T} ) = \max_{A\subset D: |A|=t} \frac{1}{2} \log  \det(I + \alpha_t^{-1} \bar K_t )$,  we can bound the term $\|S_t\|_{\tilde V^{-1}_t}$ as follows,
	\begin{lemma}
		With probability at least $1-\delta$, the following holds simultaneously over all $t>0$:
		\begin{align}
		\|S_t\|_{\tilde V^{-1}_t} \leq R \sqrt{  2\log(\frac{1}{\delta}) + 2 \bar \gamma_t },
		\end{align}
		where $     \bar \gamma_t  = \max_{A\subset D: |A|=t} \frac{1}{2} \log \det(I_t + \alpha_t^{-1} \bar K_t )$.
	\end{lemma}
	We would like to highlight this bound is in terms of double-weighted kernel matrix $\bar \gamma_t$ instead of weighted kernel matrix $\tilde \gamma_t$.

	\subsubsection{Proof of Theorem \ref{thm: confidence bound stationary}}
	We would provide the detailed proof of Theorem \ref{thm: confidence bound stationary} in this section.
	\begin{proof}
		As $\hat{\theta}_t$ is the regularized weighted least-squares estimator of $\theta^*$ at time t in \cite{russac2019weighted} and $V_t= \sum_{s=1}^t w_s \varphi(x_s) \varphi(x_s)^T + \lambda_t I_H$, we have
		\begin{align*}
		\tilde \mu_t(x) &=  \varphi(x)^T \hat{\theta}_t = \varphi(x)^T V_t^{-1}  \sum_{s=1}^t w_s \varphi(x_s) y_s \\
		&= \varphi(x)^T V_t^{-1} [\sum_{s=1}^t \big(w_s \varphi(x_s) f^*(x_s) + w_s \varphi(x_s) \epsilon_s\big) ] \\
		&= \varphi(x)^T V_t^{-1} [\sum_{s=1}^t w_s \varphi(x_s) \varphi(x_s)^T \theta^* +\lambda_t \theta^* - \lambda_t \theta^* +S_t]\\
		&= \varphi(x)^T \theta^* - \lambda_t \varphi(x)^T V_t^{-1} \theta^* + \varphi(x)^T V_t^{-1} S_t.
		\end{align*}
		We have $ \tilde \mu_t(x)-f^*(x)= \tilde \mu_t(x) -  \varphi(x)^T \theta^*=  \varphi(x)^T V_t^{-1} S_t  - \lambda_t \varphi(x)^T V_t^{-1} \theta^*$, therefore
		\begin{align*}
		|\tilde \mu_t(x)-f^*(x)| &\leq 
		\|\varphi(x)\|_{V_t^{-1} \tilde V_t V_t^{-1}} \Big (\| V_t^{-1} S_t \|_{V_t \tilde V_t^{-1} V_t}+ \|\lambda_t V_t^{-1} \theta^*\|_{V_t \tilde V_t^{-1} V_t} \Big) \\
		&\leq \|\varphi(x)\|_{V_t^{-1} \tilde V_t V_t^{-1}} \Big (\| S_t \|_{ \tilde V_t^{-1}}+ \lambda_t \|  \theta^* \|_{ \tilde V_t^{-1} } \Big)
		\end{align*}
		
		Knowing that $\tilde V_t \succeq \alpha_t I_\mathcal{H}$ and $\tilde V_t$ is positive definite, we have $ \|\theta^* \|_{ \tilde V_t^{-1}} \leq \frac{1}{\sqrt{\alpha_t}} \|\theta^* \|_2$. With  $\acute \sigma^2_t(x)= {\lambda} \|\varphi(x)\|^2_{V_t^{-1} \tilde V_t V_t^{-1}}$, we have
		\begin{align*}
		|\tilde \mu_t(x)-f^*(x)| \leq  \frac{\acute \sigma_t(x)}{\sqrt{\lambda}}\Big( \| S_t \|_{ \tilde V_t^{-1}}+ \frac{\lambda_t}{\sqrt{\alpha_t}} \|  \theta \|_2 \Big).
		\end{align*}
		Given $\|f^*\|_H = \| \theta\|_2 \leq B$, we have
		\begin{align*}
		|\tilde \mu_t(x)-f^*(x)|  &\leq \frac{\acute \sigma_t(x)}{\sqrt{\lambda}}  \Big(\| S_t \|_{ \tilde V_t^{-1}}+ \frac{\lambda_t}{\sqrt{\alpha_t}} B \Big) \\
		&  \leq \frac{\lambda_t}{\sqrt{\lambda \alpha_t}} \acute \sigma_t(x) B+  \frac{\acute \sigma_t(x)}{\sqrt{\lambda}}  R \sqrt{  2\log(\frac{1}{\delta}) +  2 \bar \gamma_t}\\
		& = \acute \sigma_t(x) B + \frac{\acute \sigma_t(x)}{\sqrt{\lambda}}  R \sqrt{  2\log(\frac{1}{\delta}) +  2 \bar\gamma_t}\\
		& = \acute \sigma_t(x) \beta_t,
		\end{align*}
		where $\beta_t =   B+  \frac{1}{\sqrt{\lambda}}  R \sqrt{  2\log(\frac{1}{\delta}) +  2 \bar \gamma_t}$.
	\end{proof}
	
	\subsubsection{Proof of Remark \ref{rmk: stationary} } \label{sec: proof 
		remark stationary }
	The following lemma shows that Equation \eqref{equ:variance} is equivalent to $\tilde{\sigma}^2_t(x)= \lambda_t ||\varphi(x)||^2_{V_t^{-1}}$.
	\begin{lemma} \label{lemma: hat sigma}
		Equation \eqref{equ:variance} is equivalent to $\tilde{\sigma}^2_t(x)= \lambda_t ||\varphi(x)||^2_{V_t^{-1}}$.
	\end{lemma}
	\begin{proof}
		As $I_\mathcal{H} -\tilde \Phi_t^T (\tilde \Phi_t \tilde \Phi_t^T + \lambda_t I_t)^{-1} \tilde \Phi_t = \lambda_t (\tilde \Phi_t^T \tilde \Phi_t + \lambda_t I_\mathcal{H})^{-1}  = \lambda_t V_t^{-1} $, 
		we have 
		\begin{align*}
		\tilde \sigma_t(x)^2 &= k(x,x) -  \tilde  k_t(x)^T (\tilde K_t + \lambda_t I_t )^{-1} \tilde  k_t(x) \\
		&=  \varphi(x)^T \varphi(x) - \varphi(x)^T \tilde \Phi_t^T (\tilde \Phi_t \tilde \Phi_t^T + \lambda_t I_t)^{-1} \tilde \Phi_t \varphi(x)\\
		&= \varphi(x)^T  [I_\mathcal{H} -\tilde \Phi_t^T (\tilde \Phi_t \tilde \Phi_t^T + \lambda_t I_t)^{-1} \tilde \Phi_t] \varphi(x) \\
		&=  \varphi(x)^T  \lambda_t V_t^{-1}  \varphi(x)  = \lambda_t ||\varphi(x)||_{V_t^{-1}}.
		\end{align*}
	\end{proof}
	
	Assume $\lambda =1 $ and $w_t=1$. Then the followings hold; $V_t= \tilde V_t$ and $\lambda_t=\lambda$, thus $\acute \sigma_t(x) = \tilde \sigma_t(x) $ and $\tilde \mu_t(x)= \mu_t(x)$. In the above Lemma \ref{lemma: hat sigma}, we have $\acute \sigma_t(x)= \sigma_t(x)= k(x,x) -k_t(x)^T (K_t + \lambda I )^{-1} k_t(x)$ and $\bar \gamma_t= \gamma_t$, which makes Theorem \ref{thm: confidence bound stationary} equivalent to \cite[Theorem 2]{chowdhury2017kernelized}.

	\subsection{Confidence bounds for non-stationary cases} \label{sec: confidence bound nonstationary}
	
	In this section we provide the relatively loose regret bound in terms of $\tilde{\sigma}_t(x)$ and then detailed proof of Theorem \ref{thm: confidence bound nonstationary} that the surrogate parameter $m_t(x)$ lies in the confidence ellipsoid defined in Theorem \ref{thm: confidence bound stationary} with high probability.
	
	First, we further restrict that $\acute \sigma_t(x) \geq 0$ and $\tilde{\sigma}_t(x) \geq 0$, then we have the following lemma.
	\begin{lemma} \label{lem: tilde sigma les hat sigma}
		If $\{w_s\}_{s=1}^t$ is increasing,  then $\acute \sigma_t(x) \leq \tilde{\sigma}_t(x)$.
	\end{lemma}
	\begin{proof}
		We recall that  $\acute \sigma^2_t(x)= \lambda \|\varphi(x)\|^2_{V_t^{-1} \tilde V_t V_t^{-1}}$ by definition and $\tilde{\sigma}^2_t(x)= \lambda_t ||\varphi(x)||^2_{V_t^{-1}}$ from Lemma \ref{lemma: hat sigma}. 
		As $V_t = \sum_{s=1}^t w_s \varphi(x_s) \varphi(x_s)^T + \lambda w_t I_\mathcal{H}$, we have 
		\begin{align*}
		\tilde V_t = \sum_{s=1}^t w_s^2 \varphi(x_s) \varphi(x_s)^T + \lambda w_t^2 I_\mathcal{H} \leq w_t \sum_{s=1}^t w_s \varphi(x_s) \varphi(x_s)^T + \lambda w_t^2 I_\mathcal{H} \leq w_t V_t.
		\end{align*}
		Therefore, $V_t^{-1} \tilde V_t V_t^{-1} \leq w_t V_t^{-1} V_t V_t^{-1} \leq w_t V_t^{-1}$ and $\acute \sigma^2_t(x) \leq \tilde{\sigma}^2_t(x)$ since $\lambda_t = \lambda w_t$.
	\end{proof}
	
	We would state the full proof of Theorem \ref{thm: confidence bound nonstationary} as follows.
	\begin{proof}[Proof of Theorem \ref{thm: confidence bound nonstationary}]
		We would obtain $\tilde \mu_t(x_t)= \varphi(x_t)^T   V^{-1}_{t} \tilde \Phi_t^T \tilde y_{1:t}$ from the definition of posterior mean $\tilde \mu_t(x)$ and proof of Lemma \ref{lem: equivalent}. Then, we would get the followings,
		\begin{align*}
		& m_t(x) - \tilde \mu_{t-1}(x)\\
		&= \varphi(x)^T V^{-1}_{t-1} [\sum_{s=1}^{t-1} \eta^{-s} \varphi(x_s)  f_s(x_s) + \lambda \eta^{-(t-1)} \theta^*_t - \sum_{s=1}^{t-1} \eta^{-s}   \varphi(x_s) y_s  ]\\
		&= \varphi(x)^T V^{-1}_{t-1} [\sum_{s=1}^{t-1} \eta^{-s} \varphi(x_s) f_s(x_s) + \lambda \eta^{-(t-1)} \theta^*_t -\sum_{s=1}^{t-1} \eta^{-s}  \varphi(x_s) f_s(x_s)-\sum_{s=1}^{t-1} \eta^{-s}  \varphi(x_s) \epsilon_s]\\
		&= -\varphi(x)^T V^{-1}_{t-1} S_{t-1} + \lambda \eta^{-(t-1)}  \varphi(x)^T V^{-1}_{t-1} \theta^*_t.
		\end{align*}
		
		Then, the distance between surrogate parameter and posterior mean is bounded as,
		\begin{align*}
		|m_t(x) - \tilde \mu_{t-1}(x)| &\leq |\varphi^T(x) V^{-1}_{t-1} S_{t-1} | + \lambda \eta^{-(t-1)} | \varphi(x)^T V^{-1}_{t-1} \theta^*_t | \\
		& \leq \|\varphi(x)\|_{V_{t-1}^{-1} \tilde V_{t-1} V_{t-1}^{-1}}   \|V_{t-1}^{-1} S_{t-1}\|_{V_{t-1} \tilde V_{t-1}^{-1} V_{t-1}} \\
		& \text{\space \space} +\lambda \eta^{-(t-1)} \|\varphi(x)\|_{V_{t-1}^{-1 }\tilde V_{t-1} V_{t-1}^{-1}} \|V_{t-1}^{-1}\theta^*_t\|_{V_{t-1} \tilde V_{t-1}^{-1} V_{t-1}}\\
		& \leq \|\varphi(x)\|_{V_{t-1}^{-1} \tilde V_{t-1} V_t^{-1}}  \|S_{t-1}\|_{\tilde V_{t-1}^{-1}} +  \lambda \eta^{-(t-1)} \|\varphi(x)\|_{V_{t-1}^{-1} \tilde V_{t-1} V_{t-1}^{-1}} \|\theta^*_t\|_{\tilde V_{t-1}^{-1}}\\
		&  \leq \frac{\acute \sigma_{t-1}(x)}{\sqrt{\lambda}} \|S_{t-1}\|_{\tilde V_{t-1}^{-1}} + \lambda_{t-1} \frac{\acute \sigma_{t-1}(x)}{\sqrt{\lambda}} \frac{1}{\sqrt{\alpha_{t-1}}}\|\theta^*_t\|_2\\
		&\leq \frac{\acute \sigma_{t-1}(x)}{\sqrt{\lambda}}  R \sqrt{  2\log(\frac{1}{\delta}) +  2 \bar \gamma_{t-1}}+  \acute \sigma_{t-1}(x) B\\
		&\leq \acute \sigma_{t-1}(x) \beta_{t-1}.
		\end{align*}
		The final two steps are because  $\|\theta^*_t\|_{\tilde V_{t-1}^{-1}} \leq \frac{1}{\sqrt{\alpha_{t-1}}} \| \theta_t^*\|_2$ and $\|\theta^*_t\|_2=\|f^*_t\|_H \leq B$. Due to the above Lemma \ref{lem: tilde sigma les hat sigma}, we obtain the following inequality, $|m_t(x) - \tilde \mu_{t-1}(x)| \leq  \tilde \sigma_{t-1}(x) \beta_{t-1}$.
	\end{proof}

	\section{Proof of Regret Bound} 
	
	In this section, we state the detailed analysis of dynamic regret of WGP-UCB (Algorithm \ref{alg:weighted}).
	As $\lambda_t = \lambda w_t$, the weighted GP regression problem is equivalent to the following problem, where the time-dependent weight is $w'_{s,t}=w_s/ w_t$ and regularization factor is time-independent. 
	\begin{align*}
	& \min_{f \in H_k(D)} \sum_{s=1}^{t-1} w'_{s,t} (y_s - f(x_s) )^2 + {\lambda} \| f \|_{\mathcal H}^2 
	\end{align*}
	
	\subsection{Approximation error} \label{sec: approximation error}
	First, we would explicitly obtain the approximation error.
	
	\begin{lemma} (QFF error \citep[Lemma 14]{chowdhury2019bayesian})
		If $D=[0,1]^d$, $k=k_{SE}$, then, 
		\begin{align*}
		\epsilon_m \leq d 2^{d-1} \frac{1}{\sqrt{2} \bar m^ {\bar m}} \big(\frac{e}{4 l^2}\big)^{\bar m}  = O(\frac{d 2^{d-1}}{\big(\bar m l^2)^{\bar m}}\big).
		\end{align*}
	\end{lemma}

	\begin{lemma}\label{lem:sigma approximation} Let $f \in H_k(D)$, $||f||_H \leq B$ and $k(x,x)\leq 1$ for all $x \in D$. Then we have
		$|\tilde{\sigma}_t(x) - \breve{\sigma}_t(x)| = O(\frac{\sqrt{\epsilon_m}}{1-\eta})$.
	\end{lemma}
	
	\begin{proof}
		First we define $a_t(x) = \tilde k_t(x) - \breve k_t(x) $, have  $||a_t(x)||_2 \leq \epsilon_m \sqrt{\frac{\lambda}{\eta^t (1-\eta)}}$ as well as $||\tilde k_t(x)||_2 \leq \sqrt{\frac{\lambda}{\eta^t (1-\eta)}}$.
		
		Similarly to the proof of Lemma 15 in \cite{chowdhury2019bayesian}, we bound the approximation error between inverse kernel matrices as,
		\begin{align*}
		& \|(\tilde K_t + \lambda_t I_t )^{-1} - (\breve K_t + \lambda_t I_t)^{-1}\|_2 \\
		&= \| (\tilde K_t + \lambda_t I_t )^{-1} \Big ( (\tilde K_t + \lambda_t I_t ) - (\breve K_t + \lambda_t I_t) \Big )  (\breve K_t + \lambda_t I_t)^{-1} \|_2\\
		& = \| (\tilde K_t + \lambda_t I_t )^{-1}  ( \tilde K_t -\breve K_t   )  (\breve K_t + \lambda_t I_t)^{-1} \|_2
		\end{align*}
		\begin{align*}
		& \leq \| (\tilde K_t + \lambda_t I_t )^{-1} \|_2  \| \tilde K_t -\breve K_t\|_2 \|  (\breve K_t + \lambda_t I_t)^{-1} \|_2 \\
		& \leq \frac{1}{\lambda_t}  \frac{\epsilon_m \lambda}{\eta^t (1-\eta)} \frac{1}{\lambda_t}= \frac{\epsilon_m \lambda}{\eta^t (1-\eta) \lambda_t^2}.
		\end{align*}
		The last inequality holds because $\| \tilde K_t -\breve K_t\|^2_2 \leq \sum_{1\leq i,j \leq t} \Big ( (k(x_i,x_j) - \breve k(x_i,x_j)) \lambda \sqrt{\eta^{-i-j}}  \Big)^2 \leq  \lambda^2 \epsilon_m^2 \sum_{1\leq i \leq t} \eta^{-i} \sum_{1\leq j \leq t} \eta^{-j} \leq \frac{\epsilon_m^2 \lambda^2 }{\eta^{2t} (1-\eta)^2}$ and
		$\| (\tilde K_t + \lambda_t I_t )^{-1} \|_2 \leq \frac{1}{\lambda_t}$.
		
		Therefore, we have
		\begin{align*}
		&|\tilde{\sigma}^2_t(x) - \breve{\sigma}^2_t(x)| \\
		&= |k(x,x) - \tilde k_t(x)^T (\tilde K_t + \lambda_t I_t)^{-1} \tilde k_t(x) -  
		\breve k(x,x) - \breve k_t(x)^T (\breve K_t + \lambda_t I_t)^{-1} \breve k_t(x)|\\
		&\leq |k(x,x) - \breve k(x,x)| + 
		|\tilde k_t(x)^T (\tilde K_t + \lambda_t I_t)^{-1} \tilde k_t(x) -
		\breve k_t(x)^T (\breve K_t + \lambda_t I_t)^{-1} \breve k_t(x)  |\\
		&\leq \epsilon_m + 
		|\tilde k_t(x)^T \Big ( (\tilde K_t + \lambda_t I_t)^{-1} - (\breve K_t + \lambda_t I_t)^{-1}\Big) \tilde k_t(x) | \\
		& + 2|a_t(x)^T (\breve K_t + \lambda_t I_t)^{-1} \tilde k_t(x)|
		+ |a_t(x)^T (\breve K_t \lambda_t I_t)^{-1} a_t(x)| \\
		&\leq \epsilon_m + 
		\| (\tilde K_t + \lambda_t I_t)^{-1} - (\breve K_t + \lambda_t I_t)^{-1}\|_2 \| \tilde k_t(x) \|^2_2 \\
		&+ 2\|a_t(x)\|_2 \|(\breve K_t + \lambda_t I_t)^{-1}\|_2 \|\tilde k_t(x)\|_2
		+ \| (\breve K_t \lambda_t I_t)^{-1}\|_2 \| a_t(x)\|^2_2 \\
		&\leq \epsilon_m + \frac{\epsilon_m \lambda}{\eta^t (1-\eta) \lambda_t^2} \frac{ \lambda}{\eta^t (1-\eta) } +
		2  \sqrt{\frac{ \lambda}{\eta^t (1-\eta) }} \epsilon_m \frac{1}{\lambda_t} \sqrt{\frac{ \lambda}{\eta^t (1-\eta) }}
		+\frac{1}{\lambda_t} \epsilon_m^2 \frac{ \lambda}{\eta^t (1-\eta) } \\
		&=O\big(\frac{ \epsilon_m}{(1-\eta)^2}\big).
		\end{align*}
		Then, the proof is completed from 
		$$
		|\tilde{\sigma}_t(x) - \breve{\sigma}_t(x)|^2 \leq |\tilde{\sigma}_t(x) + \breve{\sigma}_t(x)| |\tilde{\sigma}_t(x) - \breve{\sigma}_t(x)| \leq |\tilde{\sigma}^2_t(x) - \breve{\sigma}^2_t(x)|.
		$$ 
	\end{proof}

	\subsection{Bound of $\sum_{t=1}^T \tilde{\sigma}_{t-1}(x_t)$} \label{sec: hat sigma}
	In this section we would describe the way to obtain the tight bound of $\sum_{t=1}^T \tilde{\sigma}_{t-1}(x_t)$
	
	\begin{lemma}\label{lem: sigma breve}
		$\sum_{t=1}^T \breve \sigma_{t-1}(x_t) \leq \sqrt{4\lambda T \breve \gamma_T + 2\lambda m T^2 \log(1/\eta)}$.
	\end{lemma}
	\begin{proof}
		Assume the feature map has a finite dimension $m$, then
		
		\begin{align*}
		&\sum_{t=1}^T \breve \sigma_{t-1}(x_t)  \le \sqrt{T \sum_{t=1}^T \breve \sigma^2_{t-1}(x_t)} \leq \sqrt{T \sum_{t=1}^T 2\lambda \log ( 1+ \frac{1}{\lambda} \breve \sigma^2_{t-1}(x_t))} \\
		& \leq \sqrt{T \sum_{t=1}^T 2\lambda  \log ( 1+  \eta^{-(t-1)}  ||\breve \varphi(x_t)||^2_{\breve V^{-1}_{t-1}})} \leq \sqrt{2\lambda T \sum_{t=1}^T   \log ( 1+  \eta^{-t}  ||\breve \varphi(x_t)||^2_{\breve V^{-1}_{t-1}})}.
		\end{align*}
		
		Due to $\breve V_t \geq \breve V_{t-1}+ \eta^{-t} \breve \varphi(x_t)  \breve \varphi(x_t)^T \geq  \breve V^{1/2}_{t-1} (I_\mathcal{H} + \eta^{-t}  \breve V^{-1/2}_{t-1}  \breve \varphi(x_t)  \breve \varphi(x_t)^T  \breve V^{-1/2}_{t-1})  \breve V^{1/2}_{t-1}$,
		we have 
		\begin{align*}
		\det( \breve V_t) &\geq \det( \breve V_{t-1})   \det(I_\mathcal{H} + \eta^{-t/2}  \breve V^{-1/2}_{t-1}  \breve \varphi(x_t) \big(\eta^{-t/2}  \breve V^{-1/2}_{t-1}  \breve \varphi(x_t) \big)^T)\\
		&\geq  \det( \breve V_{t-1})    (1+  \eta^{-t}  || \breve \varphi(x_t)||^2_{ \breve V^{-1}_{t-1}}), 
		\end{align*}
		and then the following bound holds.
		\begin{align*}
		\sum_{t=1}^T   \log ( 1+  \eta^{-t}  || \breve \varphi(x_t)||^2_{ \breve V^{-1}_{t-1}}) &\leq \sum_{t=1}^T   \log (\frac{\det( \breve V_t)}{\det( \breve V_{t-1}) })
		\leq    \log (\Pi_{t=1}^T \frac{\det( \breve V_t)}{\det( \breve V_{t-1}) }) \leq \log (\frac{\det( \breve V_T)}{\det( \breve V_{0})}).
		\end{align*}
		
		From matrix determinant lemma stating $\det (A+UV^T)= \det (I+ V^TA^{-1}U) \det(A)$ and $V_0=\lambda I_\mathcal{H} \in \Re^{m\times m}$, $\det( \breve V_T)$ is decomposed as,
		\begin{align*}
		\det( \breve V_T) &= \det( \breve \Phi_T^T \breve \Phi_T + \lambda_T I_\mathcal{H}) =\det(I_t+ \breve \Phi_T (\lambda_T I_\mathcal{H})^{-1} \breve \Phi^T_T ) \det(\lambda_T I_\mathcal{H})\\
		&=\det(I_t+ \lambda_T^{-1} \breve \Phi_T  \breve \Phi^T_T ) \det(\eta^{-T} V_0)=\det(I_t+ \lambda_T^{-1} \breve \Phi_T  \breve \Phi^T_T ) \eta^{-mT} \det( V_0)
		\end{align*}
		
		Thus, we get 
		\begin{align*}
		\log (\frac{\det( \breve V_T)}{\det( \breve V_{0})}) &= \log \det(I_t+\lambda_T^{-1} \breve \Phi_T  \breve \Phi^T_T ) +mT \log(\eta^{-1}) \leq 2 \breve \gamma_T + mT \log(1/\eta)
		\end{align*}
		
		Therefore, we have
		$$
		\sum_{t=1}^T \breve \sigma_{t-1}(x_t) \leq \sqrt{2\lambda T \sum_{t=1}^T   \log ( 1+  \eta^{-t}  ||\breve \varphi(x)||^2_{ \breve V^{-1}_{t-1}})} \leq \sqrt{4\lambda T \breve \gamma_T + 2\lambda m T^2 \log(1/\eta)}.
		$$
	\end{proof}
	By combining Lemma \ref{lem:sigma approximation} and \ref{lem: sigma breve}, we have the following lemma.
	\begin{lemma} \label{lem: bound of sum hat sigma}
		$$ 
		\beta_T \sum_{t=1}^T \tilde{\sigma}_{t-1}(x_t) \leq   \beta_T \sqrt{4\lambda T \breve \gamma_T + 2\lambda m T^2 \log(1/\eta)} + \frac{ \beta_T T \sqrt{\epsilon_m}}{1-\eta}.
		$$
	\end{lemma}

	\subsection{Bound of $\eta^{-1/2} \sum_{s=t-c}^t \tilde \sigma_{s-1}(x_s)$}
	In this section we would describe the way to obtain the tight bound of partial sum $\eta^{-1/2} \sum_{s=t-c}^t \tilde {\sigma}_{s-1}(x_s)$.
	\begin{lemma}\label{lem:partial sigma breve}
		$ \eta^{-1/2} \sum_{s=t-c}^t \breve \sigma_{s-1}(x_s)\leq \sqrt{4\lambda c \breve \gamma_t + 2\lambda m c^2 \log(1/\eta)}$.
	\end{lemma}
	
	\begin{proof}
		Similarly to the proof of Lemma \ref{lem: sigma breve}, we get the following bound.
		\begin{align*}
		&\eta^{-1/2} \sum_{s=t-c}^t \breve \sigma_{s-1}(x_s) \leq \sqrt{c \sum_{s=t-c}^t \eta^{-1} \breve \sigma^2_{s-1}(x_s)}
		\leq \sqrt{c \sum_{s=t-c}^t 2 \lambda \log \big( 1+ \lambda^{-1} \eta^{-1} \breve \sigma^2_{s-1}(x_s)\big)}\\
		& \leq  \sqrt{4 c \lambda \sum_{s=t-c}^t \frac{1}{2}  \log \big( 1+ \lambda^{-1} \eta^{-1} \breve \sigma^2_{s-1}(x_s)\big)}
		\leq \sqrt{4 \lambda c \sum_{s=t-c}^t  \frac{1}{2} \log ( 1+  \eta^{-1} \eta^{-s+1}  ||\breve \varphi(x_s)||^2_{\breve V^{-1}_{s-1}})}\\
		&\leq \sqrt{4 \lambda c \sum_{s=t-c}^t  \frac{1}{2} \log ( 1+  \eta^{-s}  ||\breve \varphi(x_s)||^2_{\breve V^{-1}_{s-1}})}
		\end{align*}
		Due to $\det(\breve V_s) \geq  \det(\breve V_{s-1})    (1+  \eta^{-s}  ||\breve \varphi(x_s)||^2_{\breve V^{-1}_{s-1}})$,
		the upper bound can be derived as,
		\begin{align*}
		\sum_{s=t-c}^t   \log ( 1+  \eta^{-s}  ||\breve \varphi(x_s)||^2_{\breve V^{-1}_{s-1}}) &\leq \sum_{s=t-c}^t   \log (\frac{\det(\breve V_s)}{\det(\breve V_{s-1}) })\leq \log (\frac{\det(\breve V_t)}{\det(\breve V_{t-c})}).
		\end{align*}

		From matrix determinant lemma, we have 
		$$
		\det(\breve V_{t-c}) = \det(I+ \lambda_{t-c}^{-1} \breve \Phi_{t-c}  \breve \Phi^T_{t-c} )  \eta^{-m(t-c)} \det( \breve V_0) \geq \eta^{-m(t-c)} \det( \breve V_0).
		$$ 
		Thus we get
		\begin{align*}
		\log (\frac{\det(\breve V_t)}{\det(\breve V_{t-c})}) &\leq \log \det(I_t+ \lambda_t^{-1} \breve \Phi_t  \breve \Phi^T_t ) + mc \log(\eta^{-1}) \leq 2 \breve \gamma_t + mc \log(1/\eta).
		\end{align*}
		
		Therefore, the partial sum $\sum_{s=t-c}^t \breve{\sigma}_{s-1}(x_s)$ is bounded as below.
		\begin{align*}
		\eta^{-1/2} \sum_{s=t-c}^t \breve \sigma_{s-1}(x_s)&\leq \sqrt{4 \lambda c \sum_{s=t-c}^t   \frac{1}{2} \log ( 1+  \eta^{-s}  ||\breve \varphi(x_s)||^2_{\breve V^{-1}_{s-1}})}
		\leq \sqrt{4\lambda c \breve \gamma_t + 2\lambda m c^2 \log(1/\eta)}.
		\end{align*}
		
	\end{proof}
	
	By combining Lemma \ref{lem:sigma approximation} and \ref{lem:partial sigma breve}, the following lemma can be derived.
	\begin{lemma} \label{lem: bound of sum tilde sigma}
		$\eta^{-1/2} \sum_{s=t-c}^t \tilde \sigma_{s-1}(x_s)\leq \sqrt{4\lambda c \breve \gamma_t + 2\lambda m c^2 \log(1/\eta)} + \frac{c \eta^{-1/2} \sqrt{\epsilon_m}}{1-\eta}$
	\end{lemma}
	

	\subsection{Preliminary results} \label{sec: proof of regret bound_pre}
	The following lemma provides the upper bound the regret of WGP-UCB (Algorithm \ref{alg:weighted}) in terms of $\tilde \sigma_{t-1}(x_t), \acute \sigma_{t-1}(x_t)$, $c$, and $\gamma$.
	
	\begin{lemma} \label{lem: preliminary regret bound}
		Let $f_t \in H_k(D)$, $\|f_t\|_H \leq B$ and $k(x,x)\leq 1$. Then, with probability at least $1-\delta$,  
		$$
		R_T  \leq 
		2 \beta_T \sum_{t=1}^T \tilde \sigma_{t-1}(x_t)
		+ \frac{2}{\lambda}  c B_T \eta^{-1/2} \sum_{s=T-c}^T \tilde \sigma_{s-1}(x_s)
		+\frac{4B\eta^c}{\lambda (1-\eta)}T, 
		$$
		where $c\geq 1$ is an integer and $0 <\eta <1$.
	\end{lemma}
	
	\begin{proof}
		The one time step regret $r_t$ is decomposed into the stationary part (first two terms) and non-stationary part (remaining terms) as
		\begin{align*}
		r_t&= f_t(x_t^*)-f_t(x_t)= m_t(x_t^*) - m_t(x_t) + f_t(x_t^*) - m_t(x_t^*) - (f_t(x_t)-m_t(x_t)).
		\end{align*}
		We bound the stationary part as,
		\begin{align*}
		m_t(x_t^*) - m_t(x_t) 
		&\leq \tilde \mu_{t-1}(x_t^*) + \beta_{t-1} \tilde \sigma_{t-1}(x_t^*) - (\tilde \mu_{t-1}(x_t) - \beta_{t-1} \tilde \sigma_{t-1}(x_t) ) \\
		&\leq \tilde \mu_{t-1}(x_t) + \beta_{t-1} \tilde \sigma_{t-1}(x_t) - (\tilde \mu_{t-1}(x_t) - \beta_{t-1} \tilde \sigma_{t-1}(x_t) ) \\
		&\leq 2 \beta_{t-1} \tilde \sigma_{t-1}(x_t),
		\end{align*}
		where the first inequality holds by Theorem \ref{thm: confidence bound nonstationary} stating $|m_t(x) - \tilde \mu_{t-1}(x)| \leq \beta_{t-1} \acute \sigma_{t-1}(x) $, and the second inequality works by the nature of UCB-type algorithm, i.e. $x_t$ defined in Equation \eqref{equ:x_t} is the arm chosen at time $t$ and thus the following holds, $\tilde \mu_{t-1}(x^*_t) + \beta_{t-1} \tilde \sigma_{t-1}(x^*_t) \leq \tilde \mu_{t-1}(x_t) + \beta_{t-1} \tilde \sigma_{t-1}(x_t)$.
		
		For the non-stationary part, we have $f_t(x)={\theta^*_t}^T \varphi(x)$ and $m(x)=\bar \theta_t^T \varphi(x)$ from Mercer theorem, and $\|f\|^2_H=\|\theta\|^2_2 \leq B$. Then, we bound the non-stationary part in terms of distance between surrogate parameter $\bar \theta_t$ and true parameter $\theta^*$ as below,
		\begin{align*}
		f_t(x_t^*) - m_t(x_t^*) - (f_t(x_t)-m_t(x_t)) &= \left<\theta^*_t-\bar \theta_t, \varphi(x^*)-\varphi(x) \right>\\
		& \leq  \|\varphi(x^*)-\varphi(x)\|_2  \|\theta^*_t-\bar \theta_t\|_2\\
		&\leq 2 \|\varphi(x)\|_2 \|\theta^*_t-\bar \theta_t\|_2  \leq 2 \|\theta^*_t- \bar \theta_t\|_2
		\end{align*}
		where the last inequality holds due to $\|\varphi(x)\|^2_2=\varphi(x)^T\varphi(x)=k(x,x)\leq 1$.
		
		As pointed out by \cite[p.4]{zhao2021non}, the statement $\lambda_{max}(V_{t-1}^{-1} \sum_{s=t-D}^p \eta^{-s} A_s A_s^T)\leq 1$ in  \cite[p.18]{russac2019weighted} is not true. We would fix this error in the following lemma \ref{lemma: fix error} (proved at the end of this subsection). We recall some definitions of $V_t$, $\theta^*_t$, and $\bar \theta_t$ as
		\begin{align*}
		V_t &= \sum_{s=1}^t w_s \varphi(x_s) \varphi(x_s)^T + \lambda_t I_H,\\
		\theta^*_t &= V_{t-1}^{-1} \sum_{s=1}^{t-1} w_s \varphi(x_s) \varphi(x_s)^T \theta^*_t + V_{t-1}^{-1} \lambda_{t-1} \theta^*_t,\\
		\bar \theta_t & = V^{-1}_{t-1} [\sum_{s=1}^{t-1} \eta^{-s} \varphi(x_s) \varphi(x_s)^T \theta^*_s  + \lambda \eta^{-(t-1)} \theta^*_t].
		\end{align*}
		
		\begin{lemma}\label{lemma: fix error}
			$$
			\big\| V_{t-1}^{-1} \sum_{s=t-c}^p \eta^{-s} \varphi(x_s) \varphi(x_s)^T\big\|_2 \le \frac{\eta^{-1/2}}{\lambda} \sum_{s=t-c}^p \tilde \sigma_{s-1}(x_s) 
			$$
		\end{lemma}
		
		Then we would bound the distance between surrogate parameter $\bar \theta_t$ and true parameter $\theta^*$ as below.
		\begin{align*}
		&\| \theta^*_t- \bar  \theta_t\|_2 = \|V_{t-1}^{-1} \sum_{s=1}^{t} w_s \varphi(x_s) \varphi(x_s)^T ( \theta^*_s-   \theta^*_t) \|_2 \\
		\leq& \|\sum_{s=t-c}^{t-1} V_{t-1}^{-1} \eta^{-s} \varphi(x_s) \varphi(x_s)^T  ( \theta^*_s- \theta^*_t)  \|_2     + \|V_{t-1}^{-1} \sum_{s=1}^{t-c-1} \eta^{-s} \varphi(x_s) \varphi(x_s)^T ( \theta^*_s -  \theta^*_t)\|_2\\
		\leq& \|\sum_{s=t-c}^{t-1} V_{t-1}^{-1} \eta^{-s} \varphi(x_s) \varphi(x_s)^T \sum_{p=s}^{t-1} ( \theta^*_p- \theta^*_{p+1})  \|_2
		+ \|\sum_{s=1}^{t-c-1} \eta^{-s} \varphi(x_s) \varphi(x_s)^T ( \theta^*_s -  \theta^*_t)\|_{V_{t-1}^{-2}}\\
		\leq& \|\sum_{p=t-c}^{t-1} V_{t-1}^{-1} \eta^{-s} \varphi(x_s) \varphi(x_s)^T \sum_{s=t-c}^p ( \theta^*_p- \theta^*_{p+1})  \|_2 
		+   \frac{1}{\lambda} \sum_{s=1}^{t-c-1} \eta^{t-1-s} \| \varphi(x_s) \varphi(x_s)^T ( \theta^*_s -  \theta^*_t)\|_2\\
		\leq&  \sum_{p=t-c}^{t-1} \| V_{t-1}^{-1} \sum_{s=t-c}^p \eta^{-s} \varphi(x_s) \varphi(x_s)^T  ( \theta^*_p- \theta^*_{p+1})  \|_2 + \frac{2B}{\lambda} \sum_{s=1}^{t-c-1} \eta^{t-1-s} \\
		\leq&  \sum_{p=t-c}^{t-1} \| V_{t-1}^{-1} \sum_{s=t-c}^p \eta^{-s} \varphi(x_s) \varphi(x_s)^T\|_2 \cdot \| ( \theta^*_p- \theta^*_{p+1})  \|_2 + \frac{2B}{\lambda} \sum_{s=1}^{t-c-1} \eta^{t-1-s} \\
		\leq& \sum_{p=t-c}^{t-1} \|  \theta^*_p- \theta^*_{p+1}\|_2 \frac{\eta^{-1/2}}{\lambda} \sum_{s=t-c}^p \tilde \sigma_{s-1}(x_s)     +  \frac{2B \eta^c}{\lambda (1-\eta)}.
		\end{align*}
		The third inequality holds by $V^{-2}_t \leq (\frac{\eta^{t-1}}{\lambda})^2 I_\mathcal{H}$, and the fourth inequality works due to $\|\theta^*\|_2 \leq B $ and $\|\varphi(x)\|^2_2 = k(x,x)\leq 1$. The last inequality holds from Lemma \ref{lemma: fix error}.

		Accordingly, we would obtain the following upper bound for non-stationary part.
		\begin{align*}
		f_t(x_t^*) - m_t(x_t^*) - (f_t(x_t)-m_t(x_t)) &\leq   2 \sum_{p=t-c}^{t-1}\|f_p- f_{p+1}\|_H  \frac{\eta^{-1/2}}{\lambda} \sum_{s=t-c}^p \tilde \sigma_{s-1}(x_s)   + \frac{4B\eta^c}{\lambda (1-\eta)}.
		\end{align*}
		
		By combining bounds for both stationary and non-stationary part, the dynamic regret is bounded as,
		\begin{align*}
		R_T&= \sum_{t=1}^T r_t\\
		&\leq 2 \beta_T \sum_{t=1}^T \tilde \sigma_{t-1}(x_t) 
		+ 2  \sum_{t=1}^T \sum_{p=t-c}^{t-1} \|f_p- f_{p+1}\|_H \frac{\eta^{-1/2}}{\lambda} \sum_{s=t-c}^p \tilde \sigma_{s-1}(x_s)  
		+ \frac{4B\eta^c}{\lambda (1-\eta)} T\\
		& \leq 
		2 \beta_T \sum_{t=1}^T \tilde \sigma_{t-1}(x_t)
		+ \frac{2}{\lambda}  c B_T \eta^{-1/2} \sum_{s=T-c}^T \tilde \sigma_{s-1}(x_s)
		+\frac{4B\eta^c}{\lambda (1-\eta)}T \\
		\end{align*}
	\end{proof}

	\begin{proof}[Proof of Lemma \ref{lemma: fix error}]
		We denote the unit ball as $\mathbb{B}(1)= \{z: \|z \|_2=1\}$ and the optimizer as $z^{\star}$.
		\begin{align*}
		& \big\| V_{t-1}^{-1} \sum_{s=t-c}^p \eta^{-s} \varphi(x_s) \varphi(x_s)^T\big\|_2 
		= \sup_{z \in \mathbb{B}(1)} \big| z^T V_{t-1}^{-1} \big( \sum_{s=t-c}^p \eta^{-s} \varphi(x_s) \varphi(x_s)^T \big) z\big| \\
		& \le \| V_{t-1}^{-1} z^{\star} \|_{V_{t-1} }      \big\| \big(\sum_{s=t-c}^p \eta^{-s} \varphi(x_s) \varphi(x_s)^T\big) z^{\star} \;\big\|_{V_{t-1}^{-1} } \\
		& \le \|z^{\star} \|_{V_{t-1}^{-1}}     \big\| \sum_{s=t-c}^p \eta^{-s} \varphi(x_s) \cdot \| \varphi(x_s) \| \cdot \| z^{\star}\| \; \big\|_{V_{t-1}^{-1} } \\
		& \le \frac{1}{\sqrt{\lambda_{t-1}}}       \big\| \sum_{s=t-c}^p \eta^{-s} \varphi(x_s) \; \big\|_{V_{t-1}^{-1} } \\
		&\le \frac{\eta^{\frac{t-1}{2}}}{\sqrt{\lambda}} \big\| \sum_{s=t-c}^p \eta^{-s} \varphi(x_s) \; \big\|_{V_{t-1}^{-1} } \text{\Big(Because $V_t \succeq \lambda_t I_\mathcal{H}$ and $\|z^* \|_{V_t^{-1}} \leq \frac{1}{\sqrt{\lambda_t}} \|z^* \|_2$\Big)} \\
		& \le  \frac{1}{\sqrt{\lambda}} \big\| \sum_{s=t-c}^p \eta^{\frac{t-1-s}{2}} \eta^{\frac{-s}{2}} \varphi(x_s) \; \big\|_{V_{t-1}^{-1} }  \\
		&\le \frac{1}{\sqrt{\lambda}} \big\| \sum_{s=t-c}^p \eta^{\frac{-s}{2}}  \varphi(x_s) \; \big\|_{V_{t-1}^{-1} } \text{\space \Big(Because $t-1-s \geq 0$ and $0<\eta<1$\Big)}\\
		& \le \frac{\eta^{-1/2}}{\sqrt{\lambda}} \sum_{s=t-c}^p  \eta^{\frac{-s+1}{2}} \big\|  \varphi(x_s) \; \big\|_{V_{s-1}^{-1} }   \text{\space \Big( Because $V_{t-1}^{-1} \preceq V_{s-1}^{-1}$ \Big) }       \\
		&\le \frac{\eta^{-1/2}}{\sqrt{\lambda}} \sum_{s=t-c}^p \frac{\tilde \sigma_{s-1}(x_s)}{\sqrt{ \lambda}}     \le \frac{\eta^{-1/2}}{\lambda} \sum_{s=t-c}^p \tilde \sigma_{s-1}(x_s).  
		\end{align*}
	\end{proof}

	\subsection{Proof of Theorem \ref{thm: regret bound}} \label{sec: proof of regret bound}
	
	By combining Lemma \ref{lem: bound of sum hat sigma} , \ref{lem: bound of sum tilde sigma} and \ref{lem: preliminary regret bound}, we would obtain the dynamic regret bound as below.  
	\begin{align*}
	R_T & \le 2 \beta_T \sqrt{4\lambda T \breve \gamma_T + 2\lambda m T^2 \log(1/\eta)}
	+ \frac{2}{\lambda}  c^{3/2} B_T \sqrt{4\lambda  \breve \gamma_t + 2\lambda m c \log(1/\eta)} \\
	& + \frac{4B\eta^c}{\lambda (1-\eta)}T
	+ \frac{2}{\lambda}   B_T \eta^{-1/2} \frac{c^2 \sqrt{\epsilon_m}}{1-\eta} + \frac{ 2 \beta_T T \sqrt{\epsilon_m}}{1-\eta},
	\end{align*}
	where $c\geq 1$ is an integer and $0 <\eta <1$.
	
	\subsection{Proof of Corollary \ref{cor: order}} \label{sec: proof of order analysis}
	
	In this section we provide the regret order analysis for WGP-UCB (Algorithm \ref{alg:weighted}).
	
	\begin{lemma}
		Let $c=\frac{\log T}{1-\eta}$ and $\bar m = \log_{4/e} (T^3 \dot \gamma_T^{3/2})$. If $B_T$ is known, by choosing $\eta = 1 - \dot \gamma_T^{-1/4} B_T^{1/2} T^{-1/2}$, the regret bound is $\tilde O(\dot \gamma_T^{7/8} B_T^{1/4} T^{3/4})$. 
	\end{lemma}
	\begin{proof}
		Similarly to \cite{russac2019weighted}, we define $\log(\frac{1}{\eta})=-\log(\eta) \sim 1-\eta := X$ and $c :=\frac{\log T}{1-\eta}= X^{-1} \log T$. By defining $X=\dot \gamma_T^{-1/4} B_T^{1/2} T^{-1/2}$ and neglecting the logarithmic factors, we analyse the terms in the following regret bound one by one. 
		
		For the first term $2 \beta_T \sqrt{T} \sqrt{4\lambda \breve \gamma_T + 2\lambda m T \log(1/\eta)}$, we have the following 
		\begin{align*}
		\beta_T &\sim \dot \gamma_T^{1/2}\\
		\sqrt{4\lambda \breve \gamma_T + 2\lambda m T \log(1/\eta)} &\sim \dot \gamma_T^{1/2} T^{1/2} X^{1/2}\\
		2 \beta_T \sqrt{T} \sqrt{4\lambda \breve \gamma_T + 2\lambda m T \log(1/\eta)} &\sim  \dot \gamma_T T X^{1/2} \sim \dot \gamma_T^{7/8} B_T^{1/4}T^{3/4}.
		\end{align*}
		
		For the second term $\frac{2}{\lambda}  c^{3/2} B_T \sqrt{4\lambda  \breve \gamma_t + 2\lambda m c \log(1/\eta)}$, we have the following
		\begin{align*}
		c^{3/2} &\sim X^{-3/2}\\
		c\log(1/\eta) &\sim 1\\
		\frac{2}{\lambda}  c^{3/2} B_T \sqrt{4\lambda  \breve \gamma_t + 2\lambda m c \log(1/\eta)} &\sim \dot \gamma_T^{1/2} B_T X^{-3/2} \sim \dot \gamma_T^{7/8} B_T^{1/4} T^{3/4}.
		\end{align*}
		
		For the third term $\frac{4B\eta^c}{\lambda (1-\eta)}T$, we have 
		\begin{align*}
		\eta^c &= e^{c\log \gamma} = e^{\frac{\log \gamma}{ 1-\eta} \log T} = e^{-\log T}=T^{-1} \\
		\frac{4B\eta^c}{1-\eta} T & \sim X^{-1} \sim \dot \gamma_T^{1/4} B_T^{-1/2} T^{1/2}.
		\end{align*}
		
		For the fourth term $\frac{2}{\lambda}   B_T \eta^{-1/2} \frac{c^2 \sqrt{\epsilon_m}}{1-\eta}$, as $(1-x)^{-1/2} \approx 1+ \frac{1}{2}x$ and $1-\eta = X$,  we have
		\begin{align*}
		c^2 &\sim X^{-2}\\
		\epsilon_m^{1/2} &\sim T^{-3/2} \dot \gamma_T^{-3/4}\\
		\eta^{-1/2} &\sim X\\
		\frac{1}{1-\eta} &\sim X^{-1}\\
		\frac{2}{\lambda}   B_T \eta^{-1/2} \frac{c^2 \sqrt{\epsilon_m}}{1-\eta}  &\sim \dot \gamma_T^{-3/4} B_T  T^{-3/2} X^{-2} \sim \dot \gamma_T^{-1/4} T^{-1/2}
		\end{align*}
		
		For the last term $\frac{ 2 \beta_T T \sqrt{\epsilon_m}}{1-\eta}$, we have $\epsilon_m = O ((e/4)^{\bar m})=O(T^{-3} \dot \gamma_T^{-3/2})$. Then we have
		\begin{align*}
		\beta_T &\sim \dot \gamma_T^{1/2}\\
		\epsilon_m^{1/2} &\sim T^{-3/2} \dot \gamma_T^{-3/4}\\
		\frac{1}{1-\eta} &\sim X^{-1}\\
		\frac{ 2 \beta_T T \sqrt{\epsilon_m}}{1-\eta} &\sim \dot \gamma_T^{-1/4} T^{-1/2} X^{-1} \sim B_T^{-1/2}.
		\end{align*}
		
		By combining five terms, we complete the regret order analysis,
		\begin{align*}
		R_T &\leq B_T^{-1/2}+ \dot \gamma_T^{-1/4} T^{-1/2}+ \dot \gamma_T^{1/4} B_T^{-1/2}T^{1/2} +\dot \gamma_T^{7/8} B_T^{1/4} T^{3/4}+  \dot \gamma_T^{7/8} B_T^{1/4} T^{3/4} =\tilde O( \dot \gamma_T^{7/8} B_T^{1/4} T^{3/4} ).
		\end{align*}
		
	\end{proof}

	\begin{lemma}
		Let $c=\frac{\log T}{1-\eta}$ and $\bar m = \log_{4/e} (T^3 \dot \gamma_T^{3/2})$. If  $B_T$ is unknown, by choosing $\eta = 1 - \dot \gamma_T^{-1/4} T^{-1/2}$, the regret bound is $\tilde O(\dot \gamma_T^{7/8} B_T T^{3/4})$.
	\end{lemma}
	
	\begin{proof}
		Similar to the proof above, we define $\log(\frac{1}{\eta})=-\log(\eta) \sim 1-\eta := X$ and $c :=\frac{\log T}{1-\eta}= X^{-1} \log T$. By defining $X=\dot \gamma_T^{-1/4}  T^{-1/2}$ and neglecting the logarithmic factors, we analyse the terms in the following regret bound one by one. 
		
		For the first term, we have the following 
		\begin{align*}
		2 \beta_T \sqrt{T} \sqrt{4\lambda \breve \gamma_T + 2\lambda m T \log(1/\eta)} &\sim  \dot \gamma_T T X^{1/2} \sim \dot \gamma_T^{7/8} T^{3/4}
		\end{align*}
		For the second term, we have the following
		\begin{align*}
		\frac{2}{\lambda}  c^{3/2} B_T \sqrt{4\lambda  \breve \gamma_T + 2\lambda m c \log(1/\eta)} &\sim \dot \gamma_T^{1/2} B_T X^{-3/2} \sim \dot \gamma_T^{7/8} B_T T^{3/4} 
		\end{align*}
		For the third term, we have 
		\begin{align*}
		\frac{4B\eta^c}{1-\eta} T & \sim X^{-1} \sim \dot \gamma_T^{1/4}  T^{1/2}
		\end{align*}
		For the fourth term, we have
		\begin{align*}
		\frac{2}{\lambda}   B_T \frac{c^2 \sqrt{\epsilon_m}}{1-\eta}  &\sim \dot \gamma_T^{-3/4} B_T T^{-3/2} X^{-3} \sim B_T
		\end{align*}
		For the last term, we have $\epsilon_m = O ((e/4)^{\bar m})=O(T^{-3} \dot \gamma_T^{-3/2})$. Then we have
		\begin{align*}
		\frac{ 2 \beta_T T \sqrt{\epsilon_m}}{1-\eta} &\sim \dot \gamma_T^{-1/4} T^{-1/2} X^{-1} \sim 1.
		\end{align*}
		
		By combining five terms, we complete the regret order analysis,
		\begin{align*}
		R_T &\leq 1 + B_T + \dot \gamma_T^{1/4} T^{1/2} + \dot \gamma_T^{7/8} B_T T^{3/4} + \dot \gamma_T^{7/8} T^{3/4}= \tilde O( \dot \gamma_T^{7/8} B_T T^{3/4}).
		\end{align*}
	\end{proof}

	\section{Proof of Weighted Information Gain}
	In this section we provide two types of upper bounds of maximum information gain.
	
	\subsection{Universal Bound} \label{sec: universal bound}
	In this section we present the proof of Theorem \ref{thm: gamma bound}.
	
	\begin{proof}[Proof of Theorem \ref{thm: gamma bound}]
		The proof is composed of two following lemmas. 
		\begin{lemma} \label{lem: bar gamma universal bound}
			$\bar \gamma_T\leq \frac{N}{2} \log \Big(1+ \frac{\dot k  T}{\lambda N } \Big) + \frac{T}{2\lambda } \delta_N$ 
		\end{lemma}
		\begin{proof}
			In the similar way as \cite[Theorem 3]{vakili2021information}, we define $T$-by-$T$ matrix $\bar K_P =[\bar k_P(x_i,x_j)]_{i,j=1}^T$ and $\bar K_O=[\bar k_O(x_i,x_j)]_{i,j=1}^T$. Then we have $\bar K_t = \bar K_P + \bar K_O$.
			
			The mutual information is decomposed into two terms.
			\begin{align*}
			\bar  I(y_t;f_t) &= \frac{1}{2} \log \det (I_t+ \alpha_t^{-1} \bar K_t )\\
			&= \frac{1}{2}\log \det (I_t+ \alpha_t^{-1} \bar K_P ) + \frac{1}{2}\log \det (I_t+ \alpha_t^{-1} (I_t+ \alpha_t^{-1} \bar K_p)^{-1} \bar K_O ).
			\end{align*}
			
			To get the tighter bound, we specify $\alpha_t = \lambda w_t^2$. The first term is bounded as, where we define $\bar K_P = \bar \Psi_N \bar C_N \bar \Psi^T_N$ and $\bar G_t= \bar C_N^{1/2} \bar \Psi_N^T  \bar \Psi_N \bar C_N^{1/2}$.
			
			\begin{align*}
			&\frac{1}{2} \log \det (I_t + \alpha_t^{-1} \bar K_P ) \leq \frac{1}{2}N \log \Big( \frac{1}{N} tr(I_N + \alpha_t^{-1} \bar G_t ) \Big )\\
			&\leq \frac{1}{2}N \log \Big( 1+ \frac{1}{N}\alpha_t^{-1} \sum_{s=1}^t \bar k_p(x_s,x_s)   \Big )
			\leq \frac{1}{2}N \log \Big( 1+ \frac{1}{N} \frac{1}{\lambda w_t^2} \sum_{s=1}^t w_s^2 k_p(x,x)   \Big )\\
			&\leq \frac{1}{2}N \log \Big( 1+ \frac{\dot k}{N} \frac{1}{\lambda} \sum_{s=1}^t \frac{w_s^2}{w_t^2}   \Big )\leq \frac{1}{2}N \log \Big( 1+ \frac{\dot k}{N} \frac{1}{\lambda} \sum_{s=1}^t 1   \Big ) \le \frac{N}{2} \log \Big(1+ \frac{\dot kt }{\lambda N } \Big).
			\end{align*}
			
			For the second term, as the largest eigenvalue of $(I_t+ \alpha_t^{-1} \bar K_P)^{-1}$ is upper bounded by 1, we have $tr((I_t+ \alpha_t^{-1} \bar K_P)^{-1}\bar K_O)\leq tr(\bar K_O)$. Then, we have
			
			\begin{align*}
			&\frac{1}{2} \log \det (I_t+ \alpha_t^{-1} (I_t+ \alpha_t^{-1} \bar K_P)^{-1} \bar K_O ) \leq \frac{t}{2} \log \Big ( \frac{1}{t}   tr(I_t+ \alpha_t^{-1} (I_t+ \alpha_t^{-1} \bar K_P)^{-1} \bar K_O )   \Big)\\
			&\leq \frac{t}{2} \log \Big ( \frac{1}{t}  (t+ \alpha_t^{-1} tr(\bar K_O) )   \Big)
			\leq \frac{t}{2} \log \Big ( \frac{1}{t}  (t+  \alpha_t^{-1} \sum_{s=1}^t \bar k_O(x_s, x_s) )   \Big)\\
			&\leq \frac{t}{2} \log \Big ( \frac{1}{t}  (t+ \frac{1}{\lambda w_t^2 }\sum_{s=1}^t w_s^2 k_O(x_s, x_s) )   \Big)
			\leq \frac{t}{2} \log \Big ( \frac{1}{t}  (t+ \frac{1}{\lambda } \sum_{s=1}^t \frac{w_s^2}{w_t^2} \delta_N )  \Big)\\
			&\leq \frac{t}{2} \log \Big ( \frac{1}{t}  (t+ \frac{1}{\lambda } \sum_{s=1}^t  \delta_N )  \Big)
			\leq \frac{t}{2} \log \Big ( \frac{1}{t}  (t+ \frac{t}{\lambda  } \delta_N )  \Big)
			\leq \frac{t}{2} \log \Big ( 1+ \frac{1}{\lambda } \delta_N  \Big) \le \frac{t}{2\lambda} \delta_N.
			\end{align*}
			Combining two terms, we provide the upper bound of maximal information gain for double weighted kernel matrix as,
			\begin{align*}
			\bar  \gamma_T \le \frac{N}{2} \log \Big(1+ \frac{\dot k  T}{\lambda N } \Big) + \frac{T}{2\lambda } \delta_N.
			\end{align*}
		\end{proof}
		
		\begin{lemma}
			$\breve \gamma_T\leq \frac{N}{2} \log \Big(1+ \frac{\dot k  T}{\lambda N } \Big) + \frac{T}{2\lambda } \delta_N$.
		\end{lemma}
		
		\begin{proof}
			In the similar way of previous lemma, by replacing $\alpha_t$ with $\lambda_t$ and $\bar K$ with $\tilde K$, we also provide the same upper bound of maximal information gain for weighted kernel matrix as,
			\begin{align*}
			\tilde  \gamma_T&\le \frac{N}{2} \log \Big(1+ \frac{\dot k  T}{\lambda N } \Big) + \frac{T}{2\lambda } \delta_N
			\end{align*}
			where $\sum_{s=1}^t \frac{w_s}{w_t} \leq \sum_{s=1}^t 1 = t$ is used.
			The result follows as $\breve \gamma_t =   \frac{1}{2} \log \det (I + \lambda_t^{-1} \breve \Phi_t \breve \Phi_t^T  )\leq  \tilde \gamma_t =  \max_{A\subset D: |A|=t} \frac{1}{2} \log \det(I + \lambda_t^{-1} W K_A W^T ))$ since $\breve \Phi_t = W [\breve \varphi(x_1), \ldots, \breve  \varphi(x_t)]^T$ and $K_t=\Phi_t \Phi_t^T$.
		\end{proof}
	\end{proof}
	
	\subsection{Weight dependent bound}
	We present the proof of Theorem \ref{thm: bar gamma bound}. To get the tighter bound, we specify the weight $w_t = \eta^{-t}$ and thus $\alpha_t = \lambda \eta^{-2t}$. 
	
	\begin{proof}[Proof of Theorem \ref{thm: bar gamma bound}]
		The proof is similar to proof of Lemma \ref{lem: bar gamma universal bound}.
		
		By replacing $w_t$ by $\eta^{-t}$ , we have
		\begin{align*}
		&\frac{1}{2} \log \det (I_t + \alpha_t^{-1} \bar K_P ) 
		\leq \frac{1}{2}N \log \Big( 1+ \frac{\dot k}{N} \frac{1}{\lambda} \sum_{s=1}^t \eta^{2t-2s}   \Big )\\
		&\leq \frac{N}{2} \log \Big(1+ \frac{\dot k (1- \eta^{2t}) }{\lambda N (1-\eta^2)} \Big)
		\le \frac{N}{2} \log \Big(1+ \frac{\dot k }{\lambda N (1-\eta^2)} \Big).
		\end{align*}
		For the second term, 
		\begin{align*}
		&\frac{1}{2} \log \det (I_t+ \alpha_t^{-1} (I_t+ \alpha_t^{-1} \bar K_P)^{-1} \bar K_O ) 
		\leq \frac{t}{2} \log \Big ( \frac{1}{t}  (t+ \frac{1}{\lambda } \sum_{s=1}^t \eta^{2t-2s} \delta_N )  \Big)\\
		&\leq \frac{t}{2} \log \Big ( \frac{1}{t}  (t+ \frac{1-\eta^{2t}}{\lambda (1-\eta^2) } \delta_N )  \Big)
		\leq \frac{t}{2} \log \Big ( 1+ \frac{{1-\eta^{2t}}}{t\lambda (1-\eta^2) } \delta_N  \Big)  \le \frac{1}{2\lambda (1-\eta^2) } \delta_N.
		\end{align*}
		Combining two terms, we provide the upper bound of maximal information gain for double weighted kernel matrix as,
		\begin{align*}
		\bar  \gamma_T &\le \frac{N}{2} \log \Big(1+ \frac{\dot k  }{\lambda N (1-\eta^2)} \Big) + \frac{1}{2\lambda (1-\eta^2) } \delta_N.
		\end{align*}
		
		In the similar way, we also provide the upper bound of maximal information gain for single weighted kernel matrix as,
		\begin{align*}
		\tilde  \gamma_T&\le \frac{N}{2} \log \Big(1+ \frac{\dot k  }{\lambda N (1-\eta)} \Big) + \frac{1}{2\lambda (1-\eta) } \delta_N.
		\end{align*}
		The result follows $\breve \gamma_T \leq \tilde \gamma_T$.
	\end{proof}
	
	We also present the proof of Corollary \ref{cor: gamma bound eigendecay}.
	\begin{proof}[Proof of Corollary \ref{cor: gamma bound eigendecay}]
		
		Under the $(C_p, \beta_p)$ polynomial eigendecay condition, we obtain the following bound on $\delta_N$ as
		\begin{align*}
		\delta_N = \sum_{m=N+1}^{\infty} \lambda_m \phi^2 \le C_p N^{1-\beta_p} \phi^2.
		\end{align*}
		By choosing $N = \lceil{\big( \frac{C_p \phi^2}{\lambda(1-\eta^2)}\big)^{\frac{1}{\beta_p}} \log^{-\frac{1}{\beta_p}} ( 1+ \frac{\dot k}{\lambda (1- \eta^2)}) } \rceil$,
		\begin{align*}
		\bar \gamma_T \le \Big( \big( \frac{C_p \phi^2}{\lambda(1-\eta^2)}\big)^{\frac{1}{\beta_p}}\log^{-\frac{1}{\beta_p}} ( 1+ \frac{\dot k}{\lambda (1- \eta^2)})+ 1 \Big) \log( 1+ \frac{\dot k}{\lambda (1- \eta^2)}).
		\end{align*}
		
		Under the $(C_{e,1}, C_{e,2}, \beta_e)$ exponential eigendecay condition, we obtain the following bound on $\delta_N$ as
		\begin{align*}
		\delta_N \le \int_{z=N}^{\infty} C_{e,1} \exp(-C_{e,2} z^{\beta_e}) \phi^2 dz.
		\end{align*}
		Now, consider the case of $\beta_e = 1$ (skip the case of $\beta_e \neq 1$ for simplicity). Then,
		\begin{align*}
		\int_{z=N}^{\infty}\exp(-C_{e,2} z^{\beta_e}) \phi^2 dz = \frac{1}{C_{e,2}} \exp( -C_{e,2} N).
		\end{align*}
		With the similar logic, we choose $N = \lceil{\frac{1}{C_{e,2}} \log \big( \frac{C_{e,1} \phi^2}{C_{e,2} \lambda(1-\eta^2)}\big)} \rceil$, then we obtain the following bound,
		\begin{align*}
		\bar \gamma_T \le \Big( \frac{1}{C_{e,2}} \big( \log (\frac{1}{1-\eta^2}) + C_{\beta_e} \big) +1 \Big)  \log( 1+ \frac{\dot k}{\lambda (1- \eta^2)}).
		\end{align*}
	\end{proof}

\end{document}